\documentclass[11pt,a4paper]{article}
\pdfoutput=1

\usepackage{fullpage}
\usepackage{amsmath}
\usepackage{amssymb}
\usepackage{natbib}
\usepackage{authblk}
\usepackage{graphicx}
\usepackage{url}
\usepackage{algorithm2e}
\usepackage{xcolor}
\usepackage[utf8]{inputenc}       %
\usepackage[T1]{fontenc}          %
\usepackage{hyperref}             %
\usepackage{cleveref}             %
\usepackage{booktabs}             %
\usepackage{amsfonts}             %
\usepackage{nicefrac}             %
\usepackage{microtype}            %
\usepackage{enumitem} 
\usepackage{mathtools}            %
\usepackage{mleftright}           %
\usepackage{nicefrac}             %
\usepackage{textcomp}
\usepackage{nicematrix}
\usepackage{tikz}

\usepackage[obeyFinal]{easy-todo} %

\usepackage{comment}             

\usepackage{amsthm}
\newtheorem{theorem}{Theorem}
\newtheorem{proposition}{Proposition}
\newtheorem{corollary}{Corollary}
\newtheorem{lemma}{Lemma}
\theoremstyle{definition}

\newtheorem{definition}{Definition}
\newtheorem{remark}{Remark}

\let\svthefootnote\thefootnote
\newcommand\blfootnote[1]{%
  \let\thefootnote\relax%
  \footnotetext{#1}%
  \let\thefootnote\svthefootnote%
}

\makeatletter

\patchcmd{\NAT@test}{\else \NAT@nm}{\else \NAT@nmfmt{\NAT@nm}}{}{}

\DeclareRobustCommand\citepos
  {\begingroup
   \let\NAT@nmfmt\NAT@posfmt%
   \NAT@swafalse\let\NAT@ctype\z@\NAT@partrue
   \@ifstar{\NAT@fulltrue\NAT@citetp}{\NAT@fullfalse\NAT@citetp}}

\let\NAT@orig@nmfmt\NAT@nmfmt
\def\NAT@posfmt#1{\NAT@orig@nmfmt{#1's}}

\makeatother

\renewcommand{\epsilon}{\varepsilon}

\newcommand{\cX}{\mathcal{X}}

\newcommand{\cP}{\mathcal{P}}

\newcommand{\fR}{\mathbb{R}}
\newcommand{\one}{\operatorname{\boldsymbol{1}}}
\newcommand{\eps}{\varepsilon}

\newcommand*{\defeq}{\stackrel{\text{def}}{=}}

\newcommand{\DG}{\operatorname{DGap}}

\newcommand{\KL}{\mathop{D_{\mathrm{KL}}}\limits}
\newcommand{\TV}{\mathop{D_{\mathrm{TV}}}\limits}

\newcommand{\Berpm}{\mathrm{Ber}^{\pm}}
\DeclareMathOperator*{\argmin}{arg\,min}
\DeclareMathOperator*{\argmax}{arg\,max}
\DeclareMathOperator*{\E}{\mathbf{E}}

\newcommand{\paren}[1]{\mleft({#1}\mright)}
\newcommand{\rbr}{\paren}    %
\newcommand{\brackets}[1]{\mleft[{#1}\mright]}
\newcommand{\sbr}{\brackets} %
\newcommand{\braces}[1]{\mleft\{{#1}\mright\}}
\newcommand{\cbr}{\braces}   %

\newcommand{\abs}[1]{\mleft|{#1}\mright|}
\newcommand{\nbr}[1]{\mleft\|{#1}\mright\|}

\newcommand\RemoveAmpNL[1]{
    \bgroup         %
    \catcode`\&=9   %
    \let\\\relax    %
    \scantokens{#1} %
    \egroup         %
}
\newcommand{\NegSpaceDot}{\kern-\nulldelimiterspace}

\DeclarePairedDelimiter{\rbrm}()
\DeclarePairedDelimiter{\sbrm}[]
\DeclarePairedDelimiter{\cbrm}\{\}
\DeclarePairedDelimiter{\absm}||
\DeclarePairedDelimiter{\ceilm}\lceil\rceil

\newcommand{\Deltar}{\Delta}
\newcommand{\Deltac}{\Delta'}
\newcommand{\Rr}{\mathrm{Reg}}
\newcommand{\Rc}{\mathrm{Reg}'}

\newcommand{\Tr}{N}
\newcommand{\Tc}{N'}
\newcommand{\lr}{\ell}
\newcommand{\lc}{\ell'}

\newcommand{\xstar}{x_{\star}}
\newcommand{\ystar}{y_{\star}}
\newcommand{\xbar}{\bar{x}}
\newcommand{\ybar}{\bar{y}}
\newcommand{\xstarset}{\mathcal{X}_{\star}}
\newcommand{\ystarset}{\mathcal{Y}_{\star}}
\newcommand{\istar}{i_{\star}}
\newcommand{\jstar}{j_{\star}}
\newcommand{\omegar}{\omega}
\newcommand{\omegac}{\omega'}
\newcommand{\Sr}{S}
\newcommand{\Sc}{S'}

\newcommand{\Logr}{L}
\newcommand{\Logc}{L'}

\newcommand{\pir}{\pi}
\newcommand{\pic}{\pi'}
\newcommand{\NEr}{\mathcal{X}_{\star}}
\newcommand{\NEc}{\mathcal{Y}_{\star}}

\newcommand{\consti}{C_1}
\newcommand{\constd}{C_2}
\newcommand{\gammar}{\gamma}
\newcommand{\gammac}{\gamma'}
\newcommand{\rhor}{\rho}
\newcommand{\rhoc}{\rho'}

\newcommand{\OPT}{\mathrm{OPT}}

\newcommand\yestag{\addtocounter{equation}{1}\tag{\theequation}} %

\makeatletter
\newcommand\usemm[1]{%
  \strip@pt\dimexpr0.3514598\dimexpr #1\relax\relax mm%
}
\newcommand\usein[1]{%
  \strip@pt\dimexpr0.013837\dimexpr #1\relax\relax in%
}
\makeatother

\title{Instance-Dependent Regret Bounds for Learning \\ Two-Player Zero-Sum Games with Bandit Feedback}

\date{February 20, 2025}

\author[1,2]{Shinji Ito}
\author[3]{Haipeng Luo}
\author[1,2]{Taira Tsuchiya}
\author[3]{Yue Wu}

\affil[1]{The University of Tokyo\\ \texttt{\{shinji,tsuchiya\}@mist.i.u-tokyo.ac.jp}}
\affil[2]{RIKEN AIP}
\affil[3]{University of Southern California\\ \texttt{\{haipengl, wu.yue\}@usc.edu}}

\hypersetup{
    pdftitle={Instance-Dependent Regret Bounds for Learning Two-Player Zero-Sum Games with Bandit Feedback},
    pdfauthor={Shinji Ito, Haipeng Luo, Taira Tsuchiya, and Yue Wu},
    pdfcreator={LaTeX}
}

\begin{document}

\maketitle
\blfootnote{Authors are listed in alphabetical order.}

\begin{abstract}%
No-regret self-play learning dynamics have become one of the premier ways to solve large-scale games in practice.
Accelerating their convergence via improving the regret of the players over the naive $O(\sqrt{T})$ bound after $T$ rounds has been extensively studied in recent years, but almost all studies assume access to exact gradient feedback.
We address the question of whether acceleration is possible under bandit feedback only and provide an affirmative answer for two-player zero-sum normal-form games.
Specifically, we show that if both players apply the Tsallis-INF algorithm of~\citet{zimmert2021tsallis}, then their regret is at most $O(c_1 \log T +  \sqrt{c_2 T})$, where $c_1$ and $c_2$ are game-dependent constants that characterize the difficulty of learning ----- $c_1$ resembles the complexity of learning a stochastic multi-armed bandit instance and depends inversely on some gap measures,
while $c_2$ can be much smaller than the number of actions when the Nash equilibria have a small support or are close to the boundary. 
In particular, for the case when a pure strategy Nash equilibrium exists, $c_2$ becomes zero, leading to an optimal instance-dependent regret bound as we show.
We additionally prove that in this case our algorithm also enjoys last-iterate convergence and can identify the pure strategy Nash equilibrium with near-optimal sample complexity.
\end{abstract}

\section{Introduction}

Since the early studies that reveal the fundamental connection between online learning and game theory~\citep{foster1997calibrated,freund1999adaptive,hart2000simple}, no-regret uncoupled learning dynamics have been heavily studied and become one of the most efficient ways for robustly learning in games and finding equilibria. 
Indeed, they are the foundation for recent AI breakthroughs such as superhuman AI for poker~\citep{bowling2015heads,moravvcik2017deepstack,brown2018superhuman,brown2019superhuman}, human-level AI for Stratego~\citep{perolat2022mastering} and Diplomacy~\citep{meta2022human}, and even alignment of large language models~\citep{jacob2023consensus,munos2024nash}.

The most basic result in this area states that a no-regret self-play learning dynamic converges to some equilibrium, with the convergence rate governed by the time-averaged regret, which is usually $O(1/\sqrt{T})$ after $T$ rounds if one directly uses worst-case $O(\sqrt{T})$ regret bounds from the adversarial online learning literature.
However, there is an extensive body of research on accelerating the convergence rate by exploiting the self-play nature and game structures to achieve lower regret, from earlier studies on two-player zero-sum games~\citep{daskalakis2011near,rakhlin2013optimization} to more recent ones on multi-player general-sum games~\citep{syrgkanis2015fast, chen2020hedging, daskalakis2021near-optimal, farina2022near, anagnostides2022uncoupled, anagnostides2022near-optimal}.
Importantly, these works all assume access to exact gradient feedback.

In contrast, far less effort has been dedicated to the more realistic setting with \emph{bandit feedback} only (that is, each player only observes their noisy payoff in each round), 
partly because improving over $O(\sqrt{T})$ regret now becomes impossible in the worst case.
To get around this barrier, researchers either relax the feedback model~\citep{rakhlin2013optimization,wei2018more} or consider different notions of regret~\citep{o2021matrix, maiti2023logarithmic}.

Nevertheless, the aforementioned barrier does not mean that one cannot achieve better than $O(\sqrt{T})$ regret \emph{in all instances}.
Indeed, in stochastic $m$-armed bandits, a problem that can be seen as a special case with one player only, even though $O(\sqrt{mT})$ regret is also unavoidable in the worst case, there are plenty of studies on obtaining \emph{instance-dependent} $o(\sqrt{T})$ regret, via e.g., the Upper Confidence Bound (UCB) algorithm~\citep{auer2002using}.
This begs the question: \emph{can we also achieve good instance-dependent regret bounds for self-play learning dynamics with bandit feedback}? 

In this work, we provide an affirmative and comprehensive answer to this question for the case of two-player zero-sum normal-form games.
Importantly, our results are built on the recent advances on \emph{best-of-both-worlds} for multi-armed bandits (MAB) that use surprisingly simple algorithms and analysis to simultaneously achieve the optimal worst-case regret in the adversarial setting and the optimal instance-dependent regret in the stochastic setting (see e.g.,~\citealp{zimmert2021tsallis}).
Our work shows that it is possible to extend such techniques to the game setting and achieve similar best-of-both-worlds phenomena, where the two worlds here refer to the case when playing against an arbitrary opponent and the case when playing against the same algorithm (self-play).
More specifically, we consider an uncoupled learning dynamic where both players simply apply the Tsallis-INF algorithm of~\citet{zimmert2021tsallis}, a well-known best-of-both-worlds algorithm for MAB, and show the following guarantees.

\begin{itemize}[leftmargin=*]
\item For general zero-sum games with possibly mixed Nash equilibria (NE), each player's regret can be bounded by two terms: an $O(c_1\log T)$ term where $c_1$ is a game-dependent constant that characterizes the difficulty of learning in this game in a way analogous to stochastic MAB, and an $O(\sqrt{c_2 T})$ term where $c_2$ is also game-dependent and can be much smaller than the number of actions (the trivial bound)  --- for example, $c_2$ is small when the support of an NE is small or when all NE are close to the boundary of the strategy space.
See Section~\ref{sec:MSNE} for the exact bounds.
We also construct an example where our algorithm provably enjoys $o(\sqrt{T})$ regret and verify it empirically in Section~\ref{sec:experiments}.

\item When specifying our results to the special case with a pure strategy NE (PSNE), our regret bound only contains the $O(c_1 \log T)$ term. In fact, the regret bound is quantitatively similar to playing two stochastic MAB instances with the expected payoff vectors being the row and the column of the game respectively that contain the PSNE.
We further prove that no reasonable algorithms can improve over this bound.
Moreover, we also show that, somewhat surprisingly, our algorithm enjoys not only average-iterate convergence but also \emph{last-iterate} convergence, a much more preferable convergence guarantee when one cares about the day-to-day behavior of the learning dynamic.
See Section~\ref{sec:PSNE} for details.

\item As a by-product of our regret guarantee, we also prove that in the case with a PSNE, after running Tsallis-INF for a certain number of rounds, the pair of the most frequently selected action of each player is the PSNE with a constant probability, which can boosted to $1-\delta$ by repeating the procedure $\log(1/\delta)$ times.
Although the sample complexity of our algorithm could be $\sqrt{m}$ factor larger than that of the optimal algorithm by~\citet{maiti2024midsearch}, we emphasize that their algorithm controls both players in a centralized and coupled manner, while ours is a decentralized and uncoupled learning dynamic that additionally enjoys no-regret guarantees (even when the opponent deviates and plays arbitrarily).
\end{itemize}

\paragraph{Related work}
A line of work that is closely related to ours studies instance-dependent sample complexity (instead of regret) for finding an NE via querying entries of a zero-sum matrix game and obtaining noisy samples~\citep{maiti2023instance, maiti2024midsearch}.
This can be seen a generalization of instance-dependent sample complexity from the best-arm identification problem (e.g.,~\citealp{jamieson2014best}) to the game setting,
while our work is a generalization of instance-dependent regret from stochastic MAB (e.g.,~\citealp{auer2002using,garivier2011kl}) to the game setting.
We note that sample complexity generally does not imply any regret guarantees,  
but the latter can be translated to the former in some cases,
and we discuss such translations and compare our bounds to~\citet{maiti2023instance, maiti2024midsearch} in Sections~\ref{sec:MSNE} and~\ref{sec:PSNE_complexity}.

Another closely related topic is dueling bandits~\citep{yue2012k}, which in fact can be seen as a special case of playing a skew-symmetric zero-sum game.
The idea of sparring, first proposed by~\citet{ailon2014reducing}, is equivalent to the self-play dynamic considered here,
and the so-called ``strong regret'' in dueling bandits coincides with the sum of the individual regret of the two players (so all our results apply directly).
Most work in dueling bandits assumes the existence of a Condorcet winner, which is equivalent to the existence of a PSNE,
and some develops instance-dependent regret that is quantitatively similar to those for stochastic MAB~\citep{yue2011beat, yue2012k, zoghi2014relative}.
Using Tsallis-INF algorithm to achieve both instance-dependent regret bounds and worst-case robustness has also been studied in~\citet{zimmert2021tsallis, saha2022versatile, saad2024weak},
and our bound for the case with a PSNE is similar to theirs and can be seen as a generalization.
For the general case when a Condorcet winner might not exist, \citet{dudik2015contextual} propose the concept of von Neumann winners, which is essentially the same as mixed NE.
Assuming a unique von Neumann winner,
\citet{balsubramani2016instance} provide an instance-dependent regret in the form of $O(\sqrt{sT})$ (ignoring additive terms that are problem-dependent), where $s$ is the size of the support of the von Neumann winner.
This bound is closely related to one instantiation of our bound for general zero-sum games, but there are several other advantages of our methods, such as a much simpler algorithm, no requirement on uniqueness, and the fact that the bound holds for both player's individual regret instead of their sum only. %

As mentioned, our results are built on the recent line of work on using simple Follow-the-Regularized-Leader algorithm to achieve best-of-both-worlds for MAB, an idea first proposed by~\citet{wei2018more} and later improved and extended to various settings (e.g.,~\citealp{rouyer2020tsallis, jin2020simultaneously, jin2021best, zimmert2021tsallis, ito2021parameter,  erez2021towards, rouyer2021algorithm, ito2022nearly, amir2022better, masoudian2022best, tsuchiya2023best, jin2024improved, jin2024no}).
Even though~\citet{zimmert2021tsallis,saha2022versatile, saad2024weak} already extend the idea to a special case of dueling bandit (which itself is a special case of zero-sum games),
our work is the first to extend it to general zero-sum games, which requires new ideas and sheds light on how to further extend these techniques to more challenging settings (e.g., multi-player general-sum games).

There is a surge of studies on understanding the last-iterate convergence of self-play learning dynamics for zero-sum games, especially for the setting with exact gradient feedback~\citep{daskalakis2019last, mertikopoulos2019learning, golowich2020tight, wei2021linear, hsieh2021adaptive, gorbunov2022last, cai2022finite, cai2024fast, cai2024accelerated}.
For the bandit setting, the convergence rate is usually much slower; see e.g.,~\citet{cai2023uncoupled} where an $\tilde{O}(1/T^{\frac{1}{6}})$ rate is obtained.
We show that for the special case with a PSNE, the simple Tsallis-INF algorithm already achieves $\tilde{O}(1/\sqrt{T})$ last-iterate convergence (albeit with some instance-dependent constant).
Unlike previous analysis, ours is solely based on analyzing the regret, which might be of independent interest.

\section{Preliminaries}
In this section, we formally describe concepts related to two-player zero-sum games, self-play learning dynamics, and our main algorithm.

\paragraph{Notations}
Throughout this paper, we will use $\log(\cdot)$ to denote base-$2$ logarithm, $\ln(\cdot)$ to denote base-$e$ logarithm, and use $\log_+ x = \max \{ 1, \log x \}$. We use $\tilde{O}(\cdot)$ to hide logarithmic factors; formally, $f(x)=\tilde{O}(g(x))$ means that there exists a positive integer $k$ such that $f(x)=O(g(x) \log^k g(x))$.

\paragraph{Two-Player Zero-Sum Normal-Form Games}
A two-player zero-sum normal-form game is defined via a payoff matrix $A \in [-1, 1]^{m \times n}$,
where $m$ and $n$ are the number of actions for the row player and the column player respectively.
When the row player plays action $i$ and the column player plays action $j$,
the entry $A(i, j) \in [-1, 1]$ is the expected reward for the row player and also the expected loss for the column player (hence zero-sum).

The players also have the option to play according to a probability distribution over their actions, or a \emph{mixed strategy}.
Let $\cP_m = \{ x \in [0, 1]^m \mid \| x \|_1 = 1 \}$ be the probability simplex of size $m$.
Given mixed strategies $x \in \cP_m$ and $y \in \cP_n$ of the row and column players,
the expected reward for the row player is given by
$x^\top A y$, which is also the expected loss for the column player.

A pair of mixed strategies $(\xstar, \ystar) \in \cP_m \times \cP_n$ is a \textit{Nash equilibrium} (NE) if
$
x^{\top} A \ystar
\le
\xstar^{ \top} A \ystar
\le
\xstar^{\top} A y 
$
hold for all $x \in \cP_m$ and $y \in \cP_n$.
The celebrated Minimax theorem \citep{vonneumann1928theory} implies that $(x_\star, y_\star)$ is an NE if and only if $x_\star \in  \xstarset = \argmax_{x} \cbrm[\big]{\min_y x^\top A y}$ and $y_\star \in \ystarset = \argmin_{y} \cbrm[\big]{\max_x x^\top A y}$.

A pure-strategy Nash equilibrium (PSNE) is a Nash equilibrium $(\xstar, \ystar)$ where both players choose a pure strategy,
i.e.,
$\xstar \in \{ 0, 1 \}^m$ and $\ystar \in \{ 0, 1 \}^n$.
A PSNE is also denoted by $(\istar, \jstar)$ where $\istar \in [m]$ and $\jstar \in [n]$ are the indices of the non-zero entries of $\xstar$ and $\ystar$,
respectively.
The \textit{duality gap} for $(\hat{x}, \hat{y}) \in \cP_m \times \cP_n$ is defined by
\begin{align}
    \label{eq:DG}
    \DG(\hat{x}, \hat{y}) = \max_{x \in \cP_m, y \in \cP_n} \left\{ x^\top A \hat{y} - \hat{x}^\top A y \right\} \ge 0,
\end{align}
which measures how far $(\hat{x}, \hat{y})$ is from a Nash equilibrium.
Indeed,
$(\xstar, \ystar)$ is a Nash equilibrium if and only if $\DG(\xstar, \ystar) = 0$.

\paragraph{Learning Dynamics with Bandit Feedback}
We consider a realistic setting where both players have no prior information about the game and repeatedly play the game with bandit feedback for $T$ rounds.
Specifically, 
in each round $t = 1, 2, \ldots, T$,
the row player chooses a mixed strategy $x_t \in \cP_m$, and the column player chooses $y_t \in \cP_n$.
They each draw their action $i_t \in [m]$ and $j_t \in [n]$ from their mixed strategy,
independently from each other. %
The nature then draws an outcome $r_t \in [-1,1]$ with expectation $\E \sbrm{r_t \mid i_t,j_t } = A(i_t,j_t)$ and reveals it to the row player as their realized reward and to the column player as their realized loss.\footnote{Our results hold for the more general setting where the observations for the two players are two different samples with mean $A(i_t, j_t)$.}
Note that this is a strongly uncoupled learning dynamic as defined by~\citet{daskalakis2011near}, where the players do not need to know the mixed strategy or the realized action of the opponent (in fact, not even their existence).
This property sets us apart from previous works such as \citet{zhou2017identify,o2021matrix} that use the realized action of both players to gain insight about the matrix $A$.

From each player's perspective, they are essentially facing an MAB problem with time-varying loss vectors: $\lr_t = -A y_t$ for the row player and $\lc_t = A^\top x_t$ for the column player, with noisy feedback for the coordinate they choose.
The standard performance measure in MAB is the (pseudo-)regret, defined for the row player and the column player respectively as

\begin{equation}
\begin{aligned}
    \Rr_T &= \max_{x \in \cP_m} \Rr_T(x),
    \quad\text{where}\;\;
    \Rr_T(x)
    =
    \E \sbrm[\bigg]{
    \sum_{t=1}^T
    (x - x_t)^\top
    A
    y_t
    },
    \\
    \Rc_T &= \max_{y \in \cP_n} \Rc_T(y),
    \quad\text{where}\;\;
    \Rc_T(y)
    =
    \E \sbrm[\bigg]{
    \sum_{t=1}^T
    x_t^\top
    A
    (y_t - y)
    }.  
\end{aligned}
    \label{eq:defR}
\end{equation}
We say that an algorithm achieves no-regret if $\Rr_T$ and $\Rc_T$ grow sublinearly as $o(T)$, which has
an important game-theoretic implication, since the duality gap of the average-iterate strategy $(\bar{x}_T, \bar{y}_T)$ where
$\bar{x}_T= \frac{1}{T}{\sum_{t=1}^T x_t}$ and
$\bar{y}_T= \frac{1}{T}{\sum_{t=1}^T y_t}$ is equal to the average regret:
\begin{align*}
    \DG(\E[\bar{x}_T], \E[\bar{y}_T])
    =
    \max_{x \in \cP_m, y \in \cP_n}
    \E \sbrm[\bigg]{
        x^\top A \rbrm[\Big]{ \frac{1}{T} \sum_{t=1}^T y_t }
        -
        \rbrm[\Big]{ \frac{1}{T} \sum_{t=1}^T x_t }^\top
        A 
        y
    }
    =
    \frac{1}{T}
    \left(
    \Rr_T
    +
    \Rc_T
    \right).
\end{align*}
 Therefore, the average-iterate strategy converges to a Nash equilibrium, with the convergence rate governed by the average regret.
By simply deploying standard adversarial MAB algorithms such as Exp3~\citep{auer2002nonstochastic},
one can obtain a convergence rate of $\tilde{O}(\sqrt{(m+n)/T})$, which is not improvable in the worst case even in this game setting~\citep{klein1999number}.
The goal of this work is thus to improve the regret/convergence rate in an instance-dependent manner.

\paragraph{Tsallis-INF Algorithm}

Throughout the paper, we let both players apply the $\frac{1}{2}$-Tsallis-INF algorithm~\citep{zimmert2021tsallis},
which is based on the Follow-the-Regularized-Leader (FTRL) framework and chooses its strategy by solving the following optimization problem:
\begin{align}
    \label{eq:Tsallis-INF}
    x_t
    =
    \argmin_{x \in \cP_m}
    \cbrm[\bigg]{
    \sum_{s=1}^{t-1}
    \hat{\lr}_s^\top x
    +
    \frac{1}{\eta_t}
    \psi(x)
    },
    \quad
    y_t
    =
    \argmin_{y \in \cP_n}
    \cbrm[\bigg]{
    \sum_{s=1}^{t-1}
    \hat{\lc}_s^\top x
    +
    \frac{1}{\eta_t}
    \psi(y)
    },
\end{align}
where $\eta_t = \frac{1}{2 \sqrt{t}}$ is the learning rate, $\psi(x) = - 2 \sum_{i =1}^m \sqrt{x(i)}$ (or $- 2 \sum_{j =1}^n \sqrt{y(j)}$ for the column player, with a slight abuse of the notation) is the $\frac{1}{2}$-Tsallis entropy regularizer, and $\hat{\lr}_s$ and $\hat{\lc_s}$ are importance-weighted (IW) unbiased estimators for the loss vector $\lr_s$ and $\lc_s$ respectively, defined via\footnote{
    Shifting the loss values uniformly does not affect the behavior of the algorithm, so the \(-1\) in this equation can be removed in implementation.
    We add it here just to ensure that these indeed serve as unbiased estimators of \(\lr_t, \lc_t\).
}
\begin{equation}
\begin{aligned}
    \hat{\lr_t}(i)=\frac{\one[i_t=i](1 - r_t)}{x_t(i)} - 1, 
    \quad
    \hat{\lc_t}(j)=
    \frac{\one[j_t=j] (1 + r_t)}{y_t(j)} - 1.
\end{aligned}
    \tag{IW}\label{eq:defIW}
\end{equation}

Tsallis-INF is an algorithm that achieves the optimal
instance-dependent bound in stochastic MAB and simultaneously the optimal worst-case bound in adversarial MAB.
Directly applying its guarantee for adversarial MAB shows that both players enjoy $\sqrt{T}$-type regret always, \emph{even when their opponent behaves arbitrarily}.
We summarize this in the following theorem and omit further mention in the rest of the paper.
On the other hand, note that one cannot directly apply the guarantee of Tsallis-INF for stochastic MAB since the players are not facing a stochastic MAB instance with a fixed expected loss vector.\footnote{In fact,
\citet{zimmert2021tsallis} provide instance-dependent regret in a setting more general than the standard stochastic setting, 
but still, that does not directly apply to the game setting, especially when a PSNE does not exist.
}
Instead, we will utilize an immediate regret bound, also summarized in the theorem below, along with the self-play nature and the zero-sum game structure to prove our results.
For completeness,
we provide the proof of this theorem in Appendix~\ref{sec:app:proof-tsallis-inf}.
\begin{theorem}[\citealp{zimmert2021tsallis}]
    \label{thm:Tsallis-INF}
    For any $x \in \cP_m$,
    the pseudo-regret of the Tsallis-INF algorithm against $x$ is bounded as follows for the row player (and similarly for the column players):
    \begin{align}\label{eq:Tsallis-INF-upper}
        \Rr_T(x)
        \le
        \min_{i^* \in [m]}
        \cbrm[\bigg]{
        \E \sbrm[\bigg]{
            C_1
            \sum_{t=1}^T
            \frac{1}{\sqrt{t}}
            \sum_{i \in [m] \setminus \{i^*\}}
            \sqrt{x_t(i)}
            -
            C_2
            \sqrt{T}
            \cdot
            D(x, x_{T+1})
        }
        },
    \end{align}
    where $C_1$ and $C_2$ are positive universal constants
    and
    $D(x', x) = 
    \sum_{i=1}^m \frac{1}{\sqrt{x(i)}}(\sqrt{x'(i)} - \sqrt{x(i)})^2
    $ 
    is the Bregman divergence associated with the $\frac{1}{2}$-Tsallis entropy.
    In particular, we always have $\Rr_T = O(\sqrt{mT})$ even if the opponent behaves arbitrarily.
\end{theorem}

We will show in Appendix~\ref{sec:app:proof-tsallis-inf} that Theorem~\ref{thm:Tsallis-INF} holds with $C_1 = 19$ and $C_2 = 2$.
It is worth noting that by using a more refined analysis similar to \citet{zimmert2021tsallis},  
the values of $C_1$ and $C_2$ could be further improved.
Additionally, replacing the IW estimator~\eqref{eq:defIW} with their more sophisticated reduced variance estimator could further improve the values of $C_1$ and $C_2$.
However, such precise analysis introduces extra terms like $O(m \log T)$,  
which unnecessarily complicates the upper bound.  
To avoid such unnecessary complexity and to handle noisy observations $r_t$,  
we provide an analysis that differs from theirs.

\section{Instance-Dependent Regret Bounds for General Zero-Sum Games}\label{sec:MSNE}
{
\allowdisplaybreaks

We now provide and discuss our main theorem for general zero-sum games,
which states two regret bounds both in the form of $c_1\log T + \sqrt{c_2T}$ for some game-dependent constants $c_1$ and $c_2$.
\begin{theorem}\label{thm:general-bound-together}
If both players follow the Tsallis-INF algorithm, then for any $x \in \cP_m$ and $y\in \cP_n$, the following two upper bounds simultaneously hold for the quantity: %
\begin{equation}
\max\cbrm[\Big]{
            \Rr_T(x)
            +
            \sqrt{T}\constd{\E\sbrm[\big]{
                D(x, x_{T+1})
            }},
            \Rc_T(y)
            +
            \sqrt{T}\constd{\E\sbrm[\big]{
                D(y, y_{T+1})
            }}}
    \tag{$\star$}\label{eq:thm2-lhs}
\end{equation}
\begin{itemize}[leftmargin=*]
\item \label{enum:main-bound-omega} \hspace{1in}
$\displaystyle
\rbr{\text{\ref{eq:thm2-lhs}}}=
O\rbrm[\Big]{\sqrt{T(|I|+|J|-2)} + 
\omegar
\log_{+}\frac{mT}{\omegar^2}
+
\omegac
\log_{+}\frac{nT}{\omegac^2}},
$

where $(\xstar, \ystar)$ is an NE with maximum support, $I$ and $J$ are the support of them respectively,
$\Deltar = \rbrm[\big]{\xstar^\top A \ystar}\one - A \ystar$, %
$\Deltac = A^\top \xstar - \rbrm[\big]{\xstar^\top A \ystar}\one$, $\omegar=\sum_{i \notin I}\frac{1}{\Deltar(i)}$, and $\omegac=\sum_{j \notin J}\frac{1}{\Deltac(j)}$
;

\item \label{enum:main-bound-gamma} \hspace{1in}
$\displaystyle
\rbr{\text{\ref{eq:thm2-lhs}}}=
O\rbrm[\bigg]{\sqrt{T}\rbrm[\Big]{
    \gammar\sqrt{\log_{+}\frac{m}{\gammar^2}}+\gammac\sqrt{\log_{+}\frac{n}{\gammac^2}}
} + 
\frac{m+n}{c}\log T
},
$

where $\gammar = \max_{\xstar \in \xstarset} \sum_{i\in [m]} \sqrt{\xstar(i)} - 1$, $\gammac = \max_{\ystar \in \ystarset} \sum_{j\in [n]} \sqrt{\ystar(j)} - 1$,
and $c >0$ is a game-dependent constant such that $\DG(x,y)\geq c \min_{\xstar\in\xstarset}\nbr{x-\xstar}_1 +c \min_{\ystar\in\ystarset}\nbr{y-\ystar}_1$ holds for all $(x,y) \in \cP_m \times \cP_n$ (which always exists).

\end{itemize}

\end{theorem}

While the key of the proof also relies on the self-bounding technique that is common in the analysis of Tsallis-INF,
some new ideas are required; see details in Appendix~\ref{sec:app:proof-thm2}.
We note that Eq.~\eqref{eq:thm2-lhs} is an upper bound on $\max\{\Rr_T(x), \Rc_T(x)\}$ since Bregman divergence is non-negative,
and we include the Bregman divergence terms in Eq.~\eqref{eq:thm2-lhs} because they are crucial for proving the last-iterate convergence result in Section~\ref{sec:last-iterate}.

In both bounds of Theorem~\ref{thm:general-bound-together}, the coefficients for $\sqrt{T}$ are smaller than the trivial bound $\max\{\sqrt{m}, \sqrt{n}\}$ and reflect the proximity of the NE to a pure strategy; specifically, $\sqrt{\abs{I}+\abs{J}-2}=\gammar=\gammac=0$ when the game has a unique PSNE. This case will be further elaborated in Section~\ref{sec:PSNE}.
More generally, the coefficient $\sqrt{\abs{I}+\abs{J}-2}$ in the first bound is small when the support of the NE is small, and this bound can be seen as a generalization of that in~\citet{balsubramani2016instance} for the special case of dueling bandits.
On the other hand, the coefficients $\gammar$ and $\gammac$ in the second bound are small when the NE are close to the boundary so that some actions have much larger weight than others.
This kind of problem dependence resembles that of~\citet{maiti2023instance} who study sample complexity of finding approximate NE in the special case of $2\times n$ games.
Indeed, their sample complexity to reach $\epsilon$ duality gap is at a high-level of order $1/\epsilon^2$ multiplied with a qualitatively similar problem-dependent constant, which exactly corresponds to our $\sqrt{T}$ regret term.

The inverse coefficients for the $\log T$ term, $\Deltar$ ($\Deltac$) and $c$, quantify the relative suboptimality of alternative actions compared to the NE. 
In particular, $\Deltar$ and $\Deltac$ are exactly the standard suboptimality gaps for a stochastic MAB instance with loss vector $-A\ystar$ and $A^\top\xstar$ respectively.
Very roughly speaking, this $\log T$ term can then be interpreted as the overhead of finding the non-support of the NE, which is relatively small and is as if playing an MAB with the opponent fixed to a minimax or maximin strategy.
On the other hand, the meaning of the inverse coefficient $c$ is less clear, but its existence is guaranteed by~\citet[Theorem~5]{wei2021linear}, and we also refer the reader to their work for more details on this constant.
It only approaches zero when a strategy sufficiently different from the NE has a disproportionately small duality gap. We demonstrate this with an example:
\begin{equation}
    A=\sbrm[\bigg]{
        \begin{array}{cc}
            0 &             3\eps \\
            1 - \eps &      2\eps
        \end{array}
    },\label{eq:example-2x2-game-matrix-simplified}
\end{equation}
where $0<\eps<\frac{1}{3}$. This game has a unique NE $\xstar=(1-3\eps,3\eps), \ystar=(\eps,1-\eps)$. Direct calculation shows that $c=\eps$ satisfies the requirement for $c$. When $\eps$ approaches zero, $\gamma\approx \sqrt{\eps}$ vanishes while $\frac 1c=\frac 1\eps$ explodes.
In particular, when $\epsilon \approx 1/T^{1/3}$, our regret bound is of order $T^{1/3}$, thus provably smaller than the worst-case $\sqrt{T}$ regret.
We will revisit this example in numerical experiments in Section~\ref{sec:experiments}.

}

\section{Games with Pure-Strategy Nash Equilibria}\label{sec:PSNE}
In this section,
we further discuss the case with a unique PSNE denoted as $(\istar, \jstar)$. 
Using the first bound in Theorem~\ref{thm:general-bound-together}, we immediately obtain the following regret bound since $|I|=|J|=1$.
\begin{corollary}\label{cor:PSNE}
For a game with a unique PSNE, if both players follow the Tsallis-INF algorithm, then the following regret bound holds:  
\begin{align*}
    \max\cbrm[\big]{\Rr_T,\Rc_T} = O \rbrm[\Big]{
    {
        \omegar
        \log_+ \frac{
            mT
        }{
            \omegar^2
        }
    }
    +
    { 
        \omegac
        \log_+ \frac{
            nT
        }{
            {\omegac}^2
        }
    }
    }
    = O \rbrm[\big]{
        \rbrm{\omegar + \omegac} \log T
    },
    \yestag\label{eq:psne-bound-formula}
\end{align*}
where $\omegar = \sum_{i \neq \istar}\frac{1}{\Deltar(i)}$, $\omegac  = \sum_{j \neq \jstar } \frac{1}{\Deltac(j)}$, $\Deltar(i)=A(\istar,\jstar)-A(i,\jstar)$, and $\Deltac(j)=A(\istar,j)-A(\istar,\jstar)$.
\end{corollary}

This is a generalization of the standard instance-dependent regret bound for stochastic MAB and also similar to those from the dueling bandit literature (e.g., \citealp{yue2012k, zoghi2014relative, saha2022versatile}).
We next show that this bound is asymptotically optimal in Section~\ref{sec:PSNE_lower_bound}.
After that, we present two other results:  the last-iterate convergence behavior of our algorithm (Section~\ref{sec:last-iterate})
and using our algorithm to identify the PSNE with high probability (Section~\ref{sec:PSNE_complexity}).

\subsection{Regret lower bound}\label{sec:PSNE_lower_bound}
In this section, 
we show that the regret bound in Corollary~\ref{cor:PSNE} is tight up to some constants.
In fact,
for any $\Deltar$ and $\Deltac$,
there exists a problem instance such that
$\liminf_{T \to \infty} \frac{\Rr_T + \Rc_T}{\log T} = \Omega( \omegar + \omegac )$
for any \textit{consistent} algorithms \citep[Definition 16.1]{lattimore2020bandit}.
Note that this lower bound is also applicable to \textit{coupled} algorithms,
i.e.,
this applies to situations where a single algorithm determines both \( i_t \) and \( j_t \) based on the observation of \( \{ r_s \}_{s<t} \).

We consider problem instances in which $r_t$ follows a Bernoulli distribution over $\{ -1, 1 \}$,
i.e.,
$r_t \sim \Berpm( A(i_t, j_t) )$
given $(i_t, j_t)$,
where $\Berpm( a )$ for a parameter $a \in [-1, 1]$ is a distribution that takes values $1$ and $-1$
with probability $(1+a)/2$ and $(1-a)/2$,
respectively.
Fix an algorithm for choosing $(i_t, j_t)$ given the observations of $\{ r_s \}_{s<t}$.

Let $\Deltar \in [0, 1/4]^m$ and $\Deltac \in [0,1/4]^n$ be such that
$\Deltar_{\istar} = 0$ and $\Deltac_{\jstar} = 0$ for some $\istar \in [m]$ and $\jstar \in [n]$.
Suppose $A$ is given by
\begin{align}
    \label{eq:defADelta}
    A = 
    \mathbf{1}_m {\Deltac}^\top 
    -
    \Deltar \mathbf{1}_n^\top .
\end{align}
Then,
$(\istar, \jstar)$ is a Nash equilibrium of the game with payoff matrix $A$
as we have
$A(\istar, \jstar) - A(i, \jstar) = \Delta(i) \ge 0$
and
$A(\istar, j) - A(\istar, \jstar) = \Delta(j) \ge 0$
for all $i \in [m]$ and $j \in [n]$.
Let $\Tr_{T,i}(A)$ and $\Tc_{T,j}(A)$ denote
the expected numbers of times the $i$-th row and $j$-th column are chosen:
\begin{align*}
    \Tr_{T,i}(A) =
    \E \sbrm[\bigg]{
    \sum_{t=1}^T
    \mathbf{1}[ i_t = i ]
    }
    =
    \E \sbrm[\bigg]{
    \sum_{t=1}^T
    x_t(i)
    },
    \quad
    \Tc_{T,j}(A) =
    \E \sbrm[\bigg]{
    \sum_{t=1}^T
    \mathbf{1}[ j_t = j ]
    }
    =
    \E \sbrm[\bigg]{
    \sum_{t=1}^T
    y_t(j)
    }.
\end{align*}
We then have the following lower bound:
\begin{theorem}
    \label{thm:RegLB}
    Suppose that there exist a function $g(m, n) > 0$ and a constant $c \in (0, 1)$ such that
    $\Rr_T + \Rc_T \le g(m,n) T^{1 - c}$ for any $\hat{A} \in [-1, 1]^{m \times n}$.
    Then,
    if $A$ is given by \eqref{eq:defADelta},
    we have
    \begin{align*}
        \Tr_{T,i}(A)
        =
        \Omega\rbrm[\bigg]{
            \frac{1}{(\Deltar(i))^2}
            \log 
            \rbrm[\Big]{
            \frac{\Deltar(i) T^c}{4 g(m,n)}
            }
        },
        \quad
        \Tc_{T,j}(A)
        =
        \Omega\rbrm[\bigg]{
            \frac{1}{(\Deltac(j))^2}
            \log 
            \rbrm[\Big]{
            \frac{\Deltar(j) T^c}{4 g(m,n)}
            }
        }
    \end{align*}
    for any $i \in [m]$ and $j \in [n]$ such that
    $\Deltar_i \neq 0$ and $\Deltac_j \neq 0$.
    Consequently,
    we have
    \begin{align*}
        \liminf_{T \to \infty} \frac{\Rr_T + \Rc_T}{\log T}
        =
        \Omega\rbrm[\bigg]{
            \sum_{\substack{i \in [m] \\\Deltar(i) > 0}} \frac{c}{\Deltar(i)}
            +
            \sum_{\substack{j \in [n]\\ \Deltac(j) > 0}} \frac{c}{\Deltac(j)}
        }
        =
        \Omega\left(
            c \cdot (\omegar + \omegac)
        \right).
    \end{align*}
\end{theorem}
\begin{remark}
    \label{rem:dueling-LB}
    \upshape
    For the special case in which $A$ is skew-symmetric and $\Deltar = \Deltac$,
    the regret lower bound for the dueling bandit problem
    \citep[Theorem 2]{komiyama2015regret} leads to the same asymptotic lower bound as Theorem~\ref{thm:RegLB} above.
    Our Theorem~\ref{thm:RegLB} is more general in that it relaxes this symmetry condition; however, the underlying idea used in their proofs are shared.
    That said, while \citet{komiyama2015regret} follow the proof structure of \citet[Theorem 1]{lai1985asymptotically}, our proof, provided in Appendix~\ref{sec:pfRegLB}, adopts a simplified analytical approach based on the Bretagnolle-Huber inequality (see, e.g.,~\citealp[Chapter 17]{lattimore2020bandit}).
\end{remark}
\begin{remark}
    \upshape
    Under the assumption that bandit algorithms are \textit{minimax optimal},
    i.e.,
    if there exists a universal constant $C$ such that
    $\Rr_T + \Rc_T \le C \sqrt{(m+n)T} $ holds for all $\hat{A}\in [-1, 1]^{m \times n}$,
    which corresponds to $g(m,n) = C\sqrt{m + n}$ and $c = 1/2$,
    we obtain the following finite-time lower bound:
    \begin{align*}
        R_T(A)
        =
        \Omega
        \rbrm[\bigg]{
            \sum_{\substack{i \in [m] \\\Deltar(i) > 0}} \frac{1}{\Deltar(i)}
            \log {
                \frac{(\Deltar(i))^2 T}{16 C^2 (m+n)}
            }
            +
            \sum_{\substack{j \in [n]\\ \Deltac(j) > 0}} \frac{1}{\Deltac(j)}
            \log {
                \frac{(\Deltac(j))^2 T}{16 C^2 (m+n)}
            }
        }.
    \end{align*}
    This matches upper bound in \eqref{eq:psne-bound-formula} up to a constant factor,
    under the conditions that
    $n = \Theta(m)$
    and that
    the values of non-zero $\Deltar(i)$'s and $\Deltac(j)$'s are equivalent up to a constant factor.
\end{remark}

\subsection{Last-iterate convergence}\label{sec:last-iterate}
Somewhat surprisingly, we show that Tsallis-INF also ensures 
the following last-iterate convergence guarantee.
\begin{proposition}
    \label{prop:last-iterate}
    For a game with a unique PSNE, if both players use the Tsallis-INF algorithm,
    the output distributions converge to the PSNE (in expectation) as follows: for any $t$,
    \begin{align}
        \E \sbr{
        D(x_{\star}, x_t)
        +
        D(y_{\star}, y_t)
        }
        =
        O \rbrm[\Big]{
            \frac{1}{\sqrt{t}}
            \rbrm[\big]{
                \omegar \log_{+}{
                    \frac{mt}{\omegar^2}
                }
                +
                \omegac \log_{+}{
                    \frac{nt}{\omegac^2}
                }
            }
        },
        \label{eq:li1}
    \end{align}
    where $D(\cdot, \cdot)$ represents the Bregman divergence associated with the $(1/2)$-Tsallis entropy.
    Consequently,
    we have
    \begin{align}
        \E \sbr{
            \sqrt{
                \DG \rbr{
                    x_t,
                    y_t
                }
            }
        }
        =
        O \rbrm[\Big]{
            \frac{1}{\sqrt{t}}
            \rbrm[\big]{
                \omegar \log_{+}{
                    \frac{mt}{\omegar^2}
                }
                +
                \omegac \log_{+}{
                    \frac{nt}{\omegac^2}
                }
            }
        }.
        \label{eq:li2}
    \end{align}
\end{proposition}
Even though $\E[\sqrt{\DG(x_t, y_t)}] = O(1/\sqrt{t})$ (ignoring other factors) only imply $\E[\DG(x_t, y_t)] = O(1/\sqrt{t})$ but not necessarily $\E[\DG(x_t, y_t)] = O(1/t)$ (so the last-iterate convergence might be slower than the average-iterate convergence),
this rate is already much better than the generic $O(1/t^{1/6})$ rate of~\citet{cai2023uncoupled} for general zero-sum games.
Our proof is also particularly simple and is in fact a simple corollary of the regret bound of Theorem~\ref{thm:general-bound-together}. \\

\begin{proof}
    Fix arbitrary $T \in \mathbb{N}$.
    From the first bound of Theorem~\ref{thm:general-bound-together},
    we have
    \begin{align*}
        &
        \Rr_T(\xstar)
        +
        \Rc_T(\ystar)
        +
        C_2 \sqrt{T} \E \left[
            D(x_{\star}, x_{T+1})
            +
            D(y_{\star}, y_{T+1})
        \right]
        =
        O \rbrm[\Big]{
        {
            \omegar
            \log_+ \frac{
                mT
            }{
                \omegar^2
            }
        }
        +
        { 
            \omegac
            \log_+ \frac{
                nT
            }{
                {\omegac}^2
            }
        }
        }.
    \end{align*}
    Since $(\xstar, \ystar)$ is a Nash equilibrium, we know that 
    $\Rr_T(\xstar) + \Rc_T(\ystar) \geq 0$.
    This implies
    \begin{align*}
        \E \left[
            D(x_{\star}, x_{T+1})
            +
            D(y_{\star}, y_{T+1})
        \right]
        &
        =
        O \rbrm[\Big]{
            \frac{1}{\sqrt{T}}
            \rbrm[\big]{
                \omegar \log_{+} {
                    \frac{mT}{\omegar^2}
                }
                +
                \omegac \log_{+} {
                    \frac{nT}{\omegac^2}
                }
            }
        },
    \end{align*}
    which completes the proof of \eqref{eq:li1}.
    The Bregmann divergence associated with $(1/2)$-Tsallis entropy 
    is bounded as
    \begin{align*}
            D( \xstar, x_t )
        =
            \sum_{i=1}^m
            \frac{1}{\sqrt{x_t(i)}}\rbrm[\Big]{
                \sqrt{\xstar(i)} - \sqrt{x_t(i)}
            }^2
        \ge
            \sum_{i \in [m] \setminus \{ \istar \}}
            \sqrt{x_t(i)}
        \ge
        \frac{1}{2}
            \sqrt{
                \| x_t - \xstar \|_1.
            }
    \end{align*}
    As $\DG$ is a $1$-Lipschitz function w.r.t.~the $L^1$ norm (Lemma~\ref{lem:DGLipschitz} in Appndix~\ref{sec:app:duality-gap}),
    we have
    \begin{align*}
        \E \sbr{
            \sqrt{
                \DG( x_t, y_t )
            }
        }
        &
        \le
        \E \sbr{
            \sqrt{
                \DG( \xstar, \ystar )
                +
                \| x_t - \xstar \|_1
                +
                \| y_t - \ystar \|_1
            }
        }
        \\
        &
        \le
        2 
        \E \sbr{
            D(\xstar, x_t)
            +
            D(\ystar, y_t)
        }.
    \end{align*}
    From this and \eqref{eq:li1},
    we obtain \eqref{eq:li2} as desired.
\end{proof}

\subsection{Sample complexity of identifying PSNE}\label{sec:PSNE_complexity}
While the main focus of our work is regret minimization, we show that our algorithm can also find the exact PSNE with high probability, again using its regret guarantee.
Specifically,
define $\Delta_{\min} = \min\cbr{ \min_{i \in [m] \setminus \{ \istar \}} \Deltar(i), \min_{j \in [n] \setminus \{ \jstar \}} \Deltac(j) }$.
We prove the following.
\begin{theorem}
    For output sequences $\{ i_t \}_{t=1}^T$ and $\{ j_t \}_{t=1}^T$ generated by the Tsallis-INF algorithm,
    let $\hat{i}_T$ and $\hat{j}_T$ be the most frequently chosen arms in these sequences,
    i.e.,
    $\hat{i}_T \in \argmax_{ i \in [m] } \absm[\big]{\{ t \in [T] \mid i_t = i \} }$ and
    $\hat{j}_T \in \argmax_{ j \in [n] } \absm[\big]{ \{ t \in [T] \mid j_t = j \} }$.
    Then,
    there exists a universal constant $\alpha > 0$ such that
    $(\hat{i}_T, \hat{j}_T) = (\istar, \jstar)$
    holds with probability at least $3/4$ for
    $T \ge \alpha \frac{\omegar + \omegac}{\Delta_{\min}}$.
\end{theorem}

\begin{proof}
    From the definition of $\hat{i}_T$
    and Markov's inequality,
    we have
    \begin{align*}
        \Pr \sbrm[\big]{ \hat{i}_T \neq i_{\star} }
        &
        \le
        \Pr \sbrm[\Big]{ \sum_{t=1}^T \mathbf{1}[i_t = i_{\star}] \le \frac{T}{2}  }
        =
        \Pr \sbrm[\Big]{ T - \sum_{t=1}^T \mathbf{1}[i_t = i_{\star}] \ge \frac{T}{2}  }
        \\
        &
        \le
        \frac{2}{T}
        \E \sbrm[\Big]{ T - \sum_{t=1}^T \mathbf{1}[i_t = i_{\star}]  }
        =
        2
        -
        \frac{2}{T}
        \E \sbrm[\Big]{ \sum_{t=1}^T x_{t}(\istar)  }
        =
        2(1 - \bar{x}_{T}(\istar)).
        \yestag\label{eq:Prx}
    \end{align*}
    As we have
        $
        \bar{x}_T \cdot
        \Deltar
        \ge
        \sum_{i \neq i_{\star}}
        \bar{x}_{t}(i)
        \Deltar(i)
        \ge
        \sum_{i \neq i_{\star}}
        \bar{x}_{T}(i)
        \Deltar_{\min}
        =
        \Deltar_{\min} ( 1 - \bar{x}_{T}(\istar)),
        $
    by combining this with \eqref{eq:Prx},
    we obtain
    $
        \Pr \sbrm[\big]{ \hat{i}_T \neq i_{\star} }
        \le
        2(1 - \bar{x}_{T}( i_{\star}))
        \le
        \frac{2}{\Deltar_{\min}}
        \bar{x}_T \cdot\Deltar.
    $
    As a similar bound holds for $\Pr \sbrm[\big]{ \hat{j}_T \neq j_{\star} }$,
    we have
    \begin{align}
        \Pr \sbrm[\Big]{ (\hat{i}_T, \hat{j}_T)  \neq ( \istar, \jstar) }
        \le
        \Pr \sbrm[\big]{ \hat{i}_T \neq i_{\star}} + \Pr\sbrm[\big]{ \hat{j}_T \neq j_{\star} }
        \le
        \frac{2}{\Delta_{\min}}
        \rbrm[\Big]{
            \bar{x}_T \cdot \Deltar
            +
            \bar{y}_T \cdot \Deltac
        }
        \label{eq:PrDelta}
    \end{align}
    From the definition of $\Deltar$
    and Theorem~\ref{thm:general-bound-together},
    we have
    \begin{align*}
        \frac{2}{\Delta_{\min}}
        \rbrm[\Big]{
            \bar{x}_T \cdot \Deltar
            +
            \bar{y}_T \cdot \Deltac
        }
        & \le
        \frac{2}{\Delta_{\min}} \frac{\Rr_T + \Rc_T}{T}
        \\
        &
        \leq 2C_3
        \rbrm[\Big]{
            {
                \frac{\omegar}{\Delta_{\min}T}
                \log_+ \frac{ mT }{ \omegar^2 }
            }
            +
            { 
                \frac{\omegac}{\Delta_{\min}T}
                \log_+ \frac{ nT }{ \omegac^2 }
            }
        } \\
        &
        \leq 2C_3
        \rbrm[\Big]{
            {
                \frac{\omegar}{\Delta_{\min}T}
                \log_+ \frac{ \Delta_{\min}T }{ \omegar }
            }
            +
            { 
                \frac{\omegac}{\Delta_{\min}T}
                \log_+ \frac{ \Delta_{\min}T }{ \omegac }
            }
        },\\
        &
        \leq 2C_3
        \rbrm[\Big]{
            {
                \frac{1}{\alpha}
                \log_+ \alpha
            }
            +
            { 
                \frac{1}{\alpha}
                \log_+ \alpha
            }
        } \leq \frac{1}{4},
    \end{align*}
    where $C_3$ is the contant factor hidden by the $O(\cdot)$ symbol in \eqref{eq:psne-bound-formula},
    and in the third inequality we used the fact that $\omegar \ge m \Delta_{\min}$ and $\omegac \ge n \Delta_{\min}$.
    The last inequality holds if we take $\alpha=8C_3+4$.
    By combining this with \eqref{eq:PrDelta},
    we obtain
    $\Pr \sbrm[\big]{ (\hat{i}_T, \hat{j}_T)  \neq ( \istar, \jstar) } \le 1/4$,
    which completes the proof.
\end{proof}
From this,
we can further boost the confidence and identify the PSNE with probability at least $(1-\delta)$ with 
$O\rbrm[\big]{\frac{\omegar + \omegac}{\Delta_{\min}} \log \rbrm{ 1/\delta } }$ samples.
More concretely,
consider repeating $S > 1$ independent trials of calculating $\hat{i}_T$.
Let $\hat{i}_{T, s}$ be the result for the $s$-th trial.
Let $\tilde{i}_{T,S} \in \argmax_{i \in [m]} \absm[\big]{\{ s \in [S] \mid \hat{i}_{T, s} = i \} }$
be the arm most frequently chosen in these $S$ trials.
We then have
\begin{align*}
    \Pr \sbrm[\big]{
    \tilde{i}_{T,S}
    \neq 
    i_{\star}
    }
    &
    \le
    \Pr \sbrm[\bigg]{
    \sum_{s=1}^S
    \mathbf{1}\sbrm[\big]{
    \hat{i}_{T, s}
    =
    i_{\star}
    }
    \le
    S/2
    }
    \\
    &
    \le
    \Pr\sbrm[\bigg]{
    \sum_{s=1}^S
    X_s
    \le
    S/2
    \mid
    X_s \sim \mathrm{Ber}(3/4),
    ~
    \mbox{i.i.d. for $s \in [S]$}
    }
    \le \exp \rbr{ \Omega( - S ) }.
\end{align*}
Hence,
for any $\delta \in (0, 1)$,
by setting $S = \Theta ( 1 / \delta )$,
we have 
$
\Pr \left[
(\tilde{i}_{T,S}, \tilde{j}_{T,S})
=
(\istar, \jstar)
\right]
\ge 1 - \delta
$.
We note that,
to perform this procedure, it is necessary to know an approximate value of $\frac{\omegar + \omegac}{\Delta_{\min}}$.

We also note that due to Lemma~\ref{lem:sqrtk-ratio},
our sample complexity is at most $O\rbrm[\big]{\sqrt{\max\cbrm{n,m}}}$ times the information-theoretic optimal,
which is $O\rbrm[\big] { 
\sum_{i \in [m] \setminus \{ \istar \}} \frac{1}{{\Deltar}^2_i}
+
\sum_{j \in [n] \setminus \{ \jstar \}} \frac{1}{{\Deltac}^2_j}
}
$ and is achieved by the Midsearch algorithm of~\citet{maiti2024midsearch}.
However, their algorithm is coupled; that is, Midsearch requires the algorithm to control both players at the same time,
while our algorithm is a no-regret uncoupled learning dynamic.

\section{Numerical Experiments}\label{sec:experiments}

\begin{figure}[t]
    \centering
    \includegraphics[width=0.8\textwidth]{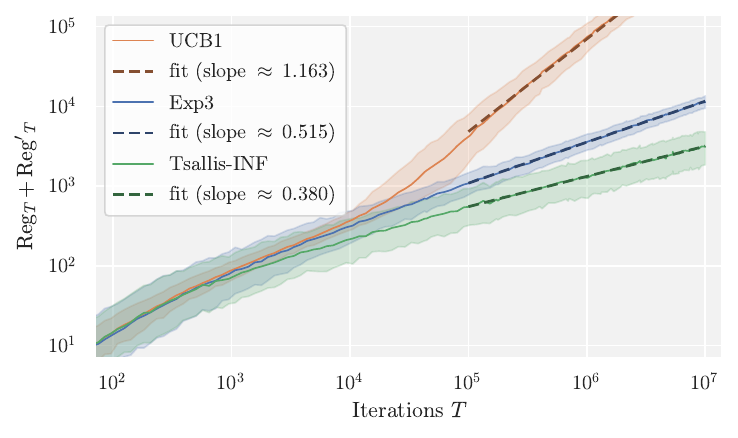}
    \caption{
    Regret scaling for Tsallis-INF and two other bandit algorithms. Each configuration $(T)$ is run for 512 trials. The interval between the 10th and 90th percentile is overlaid. The thicker dashed line represents a linear fit on the $T\geq 10^5$ subset of the log-log data.}
    \label{fig:regret-comparison}
\end{figure}

To validate our theoretical results,
we conduct a few numerical experiments.

The first experiment compares
Tsallis-INF against two baselines in terms of the regret: 
the classical UCB1~\citep{auer2002finite} and Exp3~\citep{auer2002nonstochastic} algorithms, 
which are known to have $O(T)$ and $\tilde{O}(\sqrt{T})$ regret bounds respectively in the adversarial setting.
We compare them on the game associated with $A$ defined in \eqref{eq:example-2x2-game-matrix-simplified},
with varying $T$ and $\eps=T^{-1/3}$,
where feedback $r_t$ follows a Bernoulli distribution over $\{ -1, 1 \}$ such that $ \E[r_t \mid i_t, j_t] = A(i_t, j_t)$.
As discussed, Theorem~\ref{thm:general-bound-together} predicts a regret of $\tilde{O}(T^{1/3})$ for Tsallis-INF.
The result of the experiment agrees with all these bounds in Figure~\ref{fig:regret-comparison},
where the asymptotic slope in the log-log plot (shown with a linear fit on the $T\geq 10^5$ region) is close to the theoretical prediction.

\begin{figure}[t]
    \centering
    \includegraphics[width=0.8\textwidth]{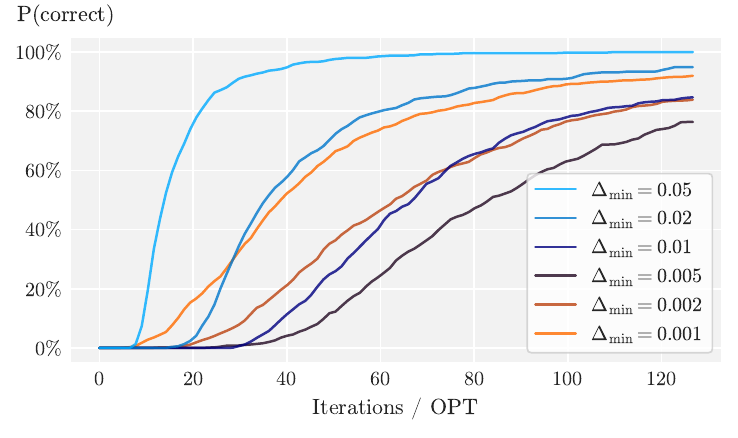}
    \caption{
        Experimental validation of Tsallis-INF's PSNE identification capability.
        The plot shows the algorithm's success rate in correctly identifying PSNE
        against the number of itrations.
        We use a hard instance of a $256\times 256$ matrix and $\Delta_1=0.1$,
        running 512 trials for each $\Delta_{\min}$ values
        over a horizon of $128\OPT$ iterations,
        where $\OPT$ is the theoretical lower bound for PSNE identification.
        The $x$-axis is scaled by $1/\OPT$.
        }
    \label{fig:PSNE-id-rate}
\end{figure}

We have discussed in Section~\ref{sec:PSNE_complexity} that Tsallis-INF needs $\frac{\omegar+\omegac}{\Delta_{\min}}$ iterations to identify the PSNE of a game. To validate our theoretical bounds, we conduct our second experiment using the following hard instance  introduced by \citet{maiti2024midsearch}:
\begin{equation}
    A=\begin{bNiceArray}{ccccc}[nullify-dots, margin, custom-line = {letter=I, tikz=dashed}, cell-space-limits = 4pt]
        0 & 2{\Delta_{\min}} & \Block{1-3}{} 2{\Delta_1} &\Cdots& 2{\Delta_1} \\
        -2{\Delta_{\min}} & \Block{4-4}{} 
                       0      & 1      & \Cdots & 1      \\
        \Block{3-1}{}
        -2{\Delta_1} & -1     & \Ddots & \Ddots & \Vdots \\
        \Vdots       & \Vdots & \Ddots & \Ddots & 1      \\
        -2{\Delta_1} & -1     & \Cdots & -1     & 0      \\
    \end{bNiceArray},
    \label{eq:psne-experiment-array}
\end{equation}
where the top-left entry is the PSNE. We set the number of actions $n=m=256$ and the gap $\Delta_1=0.1$, and vary the value of $\Delta_{\min}$.
Let $\OPT$ represent the theoretical optimal bound for identifying PSNE (ignoring log terms), defined as 
$\OPT=
\sum_{i \in [m] \setminus \{ \istar \}} \frac{1}{{\Deltar}^2_i}
+
\sum_{j \in [n] \setminus \{ \jstar \}} \frac{1}{{\Deltac}^2_j}
$,
which simplifies to $\frac{1}{2\Delta_{\min}^2}+\frac{m-2}{2\Delta_1^2}$ in this experiment.
\citepos{maiti2024midsearch} achieve the optimal $\tilde{O}(\OPT)$ sample complexity,
and their Figure~2 suggests that the sample complexity of Tsallis-INF divided by $\OPT$ is unbounded as $\Delta_{\min}$ decreases,
but our analysis in Section~\ref{sec:PSNE_complexity} disagrees with this trend.
As shown in Figure~\ref{fig:PSNE-id-rate},
the number of iterations needed to identify the PSNE divided by $\OPT$
decreases and then increases
as $\Delta_{\min}$ varies.
Lemma~\ref{lem:sqrtk-ratio} predicts the minimum ratio occurs when $\frac{\Delta_{\min}}{\Delta_1}=\frac{1}{\sqrt{m}+1}=1/17$,
and among the values we tested,
the minimum is reached when $\frac{\Delta_{\min}}{\Delta_1}=0.005/0.1=1/20$,
closely matching the prediction.
This supports our derived bound of $\tilde{O}\rbrm[\big]{\sqrt{m}\cdot \OPT}$.

The code for reproducing the experiments is available on 
\url{https://github.com/EtaoinWu/instance-dependent-game-learning}.

\section{Conclusions}\label{sec:conclusions}
Prior work on learning in games has primarily focused on the full-information setting, where each player perfectly learns their gradient.
In this paper, we investigate the more realistic, partial information setting where only a noisy version of a single realized reward is revealed to the players.
Although it is impossible to optimize for all inputs,
we demonstrated that Tsallis-INF, an existing best-of-both-worlds optimal bandit algorithm, enjoys improvements by exploiting easier instances with larger gaps.
These improvements cover three aspects: regret minimization that bounds long-term average performance, last-iterate convergence guarantees that ensure day-to-day behavior for myopic agents, and a simple way to identify PSNE.

Several important questions remain open. A natural step forward is generalizing our work to general-sum multiplayer games, which is not yet explored due to the added intricacy from the misaligned incentives of players. Equally important is extending our results to extensive-form games and continuous games, each of which introducing extra challenges with their structural complexity. Our algorithm leaves a $O(\sqrt{\max\{m,n\}})$ gap in sample complexity for pure strategy Nash equilibrium identification, and whether this gap is unavoidable in uncoupled learning dynamics remains unknown. More broadly, our work suggests that learning to play games under uncertainty may be more achievable with instance-dependent improvements, opening up new possibilities in other game-theoretic environments.

\bibliographystyle{abbrvnat}
\bibliography{reference.bib}

\newpage
\appendix

\section{Technical Lemmas}\label{sec:app:lemmas}

\begin{lemma}\label{lem:sqrt-bound-ineq}
    For any real numbers $a,b,c$, if $a>0, c>0$, then $a\leq b+\sqrt{a c}$ implies $a\leq 2b+c$.
\end{lemma}
\begin{proof}
    The function $g(t) = t-\sqrt{t}-\frac 12(t-1)$ defined on $\fR_{\geq 0}$ is convex and has a minimum of $0$ at $t=1$. This implies that
    \[
        a-\sqrt{ac}-\frac 12\rbr{a-c}=c \cdot g\rbr{\frac{a}{c}}\geq 0.
    \]
    Therefore, when $b\geq a-\sqrt{ac}$, we also have $2b\geq a-c$.
\end{proof}

\begin{lemma}
    \label{lem:sumszt-2}
    Let $a, b > 0$
    and suppose that $0\le z_t \le b$ holds for all $t = 1, 2, \ldots, T$.
    Also, assume that $\sum_{t=1}^T z_t^2\leq a$.
    We then have
    \begin{align}
        \sum_{t=1}^T \frac{z_t}{\sqrt{t}}
        \le
        f\rbr{a,b},
        \quad
        \mbox{where}
        \quad
        f(a, b) :=
            \min_{s \in [T]} \cbrm[\Big]{
            \sqrt{ a \log \frac{T}{s} } + 2 b \sqrt{s}
            }. \label{eq:sumszt-2}
    \end{align}
    Note that $f$ is a concave function.
\end{lemma}

\begin{proof}
    We split the sum into the first $n$ terms and the last $T-s$ terms:
    \begin{align}
        \sum_{t=1}^T \frac{z_t}{\sqrt{t}}
        =
        \sum_{t=1}^{s} \frac{z_t}{\sqrt{t}}
        +
        \sum_{t=s+1}^{T} \frac{z_t}{\sqrt{t}}.\label{eq:lem2-sum-two-parts}
    \end{align}
    The first part can be bounded with the fact that $z_t\leq b$:
    \begin{align*}
        \sum_{t=1}^{s} \frac{z_t}{\sqrt{t}}
        & \leq
        \sum_{t=1}^{s} \frac{b}{\sqrt{t}}\\
        & =
        b \sum_{t=1}^{s} \frac{1}{\sqrt{t}} \\
        & \leq
        b \cdot 2\sqrt{s}.\yestag\label{eq:lem2-sum-part-1}
    \end{align*}
    The other part can be bounded with the Cauchy-Schwarz inequality:
    \begin{align*}
        \sum_{t=s+1}^{T} \frac{z_t}{\sqrt{t}}
        & =
        \sum_{t=s+1}^{T} z_t\frac{1}{\sqrt{t}} \\
        & \leq
        \sqrt{\rbrm[\bigg]{\sum_{t=s+1}^{T} z_t^2}\rbrm[\bigg]{\sum_{t=s+1}^{T} \frac{1}{t}}} \\
        & \leq
        \sqrt{a \log{\frac{T}{s}}}.\yestag\label{eq:lem2-sum-part-2}
    \end{align*}
    Note that our $\log(\cdot)$ is of base 2. Our desired inequality can be obtained by plugging~\eqref{eq:lem2-sum-part-1} and \eqref{eq:lem2-sum-part-2} into \eqref{eq:lem2-sum-two-parts}.
\end{proof}

\begin{lemma}
    \label{lem:sqrtk-ratio}
    For $n\geq 2$ and an arbitrary sequence of nonnegative real numbers $x_1,\dots,x_n$, the following inequality holds:
    \begin{equation*}
        \max\cbr{x_1,\dots,x_n}\sum_{i=1}^n x_i \leq \rbrm[\Big]{\frac{1}{2}+\frac{1}{2}\sqrt{n}} \sum_{i=1}^n x_i^2.
    \end{equation*}
    Equality holds when all but one $x_i$ values are equal to $\frac{1}{\sqrt{n}+1}$ times the single outlier.
\end{lemma}
\begin{proof}
    Without loss of generality we assume that $x_1\geq x_2\geq \dots \geq x_n$. Let $\bar{x}=\frac{1}{n-1} \sum_{i=2}^n x_i$ be the average of the last $n-1$ numbers. Due to the strict convexity of the square function, we see from Jensen's inequality that
    \begin{equation}
        \sum_{i=1}^n x_i^2 = x_1^2 + \sum_{i=2}^n x_i^2 \geq x_1^2 + (n-1) \bar{x}^2,\label{eq:square-bounded-by-average}
    \end{equation}
    with equality when all $x_2, \dots, x_n$ are equal to $\bar{x}$. Consider the function $f(k)=\frac{1+(n-1)k}{1+(n-1)k^2}$ defined on $k\in [0, 1]$. Note that
    \begin{equation}
        f(k)
        =    \frac{1+(n-1)k}{1+(n-1)k^2}
        =    1+\frac{n-1}{\frac{1}{k}-1+\frac{n}{\frac{1}{k}-1}+2}
        \leq 1+\frac{n-1}{2\sqrt{n}+2}
        =    \frac12 \rbrm[\big]{\sqrt{n}+1}\label{eq:fk-bounded-by-sqrtn} 
    \end{equation}
    where the inequality is due to AM-GM, and is tight when $k=\frac{1}{\sqrt{n}+1}$.
    If we plug in $k=\frac{\bar{x}}{x_1}$, we get
    \begin{align*}
        1+(n-1)\frac{\bar{x}}{x_1}
        &\leq
        \frac12 \rbrm[\big]{\sqrt{n}+1} \rbrm[\Big]{1+(n-1)\rbrm[\big]{\frac{\bar{x}}{x_1}}^2}.\\
    \end{align*}
    We multiply both sides by $x_1^2$ and get
    \begin{align*}
        x_1(x_1+(n-1)\bar{x})
        &\leq
        \frac12 \rbrm[\big]{\sqrt{n}+1} \rbrm[\big]{x_1^2+(n-1)\bar{x}^2},\\
    \end{align*}
    which means that
    \begin{align*}
        x_1\sum_{i=1}^n x_i
        &\leq
        \frac12 \rbrm[\big]{\sqrt{n}+1} \rbrm[\big]{x_1^2+(n-1)\bar{x}^2}.\\
    \end{align*}
    Together with \eqref{eq:square-bounded-by-average} this completes the proof.
    Equality holds when both \eqref{eq:square-bounded-by-average} and \eqref{eq:fk-bounded-by-sqrtn} are tight, i.e., $x_2=\dots=x_n=\bar{x}=\frac{1}{\sqrt{n}+1}x_1$.
\end{proof}

\section{Proof of Theorem \ref{thm:Tsallis-INF}}\label{sec:app:proof-tsallis-inf}

This appendix provides the proof of Theorem~\ref{thm:Tsallis-INF}.
We include the proof of Theorem~\ref{thm:Tsallis-INF},  
as the corresponding proof is not provided in \citet{zimmert2021tsallis},  
and several aspects differ from their setting:  
the range of the loss is different,  
noisy feedback $r_t$ such that $\E[r_t \mid i_t, j_t] = A(i_t, j_t)$ is observed,  
and a negative term is introduced to ensure last-iterate convergence.

We begin by providing the following standard regret upper bound of FTRL.
By refining an analysis of FTRL,
we obtain a negative term of $- \frac{1}{\eta_{T+1}} D (x^*, x_{T+1} )$,
which is used to provide the last-iterate convergence result of  Proposition~\ref{prop:last-iterate}.
\begin{lemma}\label{lem:ftrl_bound}
    Let $\mathcal{X} \in \mathbb{R}^n$ be a non-empty compact convex set. 
    Let $\psi$ be a continuously differentiable convex function over $\mathcal{X}$.
    Suppose that $x_t$ is given by FTRL with the regularizer function $\psi$ 
    and learning rates $\eta_1 \ge \eta_2 \ge \cdots \ge \eta_{T+1} > 0$,
    as follows:
    \begin{align}
        \label{eq:defFTRL}
        x_{t} \in \argmin_{ x \in \cX } \left\{ 
        \sum_{s=1}^{t-1} \ell_s^\top x
        +
        \frac{1}{\eta_t} \psi(x)
        \right\}.
    \end{align}
    Then,
    for any $x^* \in \mathcal{X}$,
    we have
    \begin{align}
        \sum_{t=1}^T  \ell_t^\top \left( x_t - x^* \right)
        &
        \le
        \sum_{t=1}^T
        \rbrm[\Big]{
            \ell_t^\top (x_t - x_{t+1})
            -
            \frac{1}{\eta_t} D(x_{t+1}, x_t)
        }
        +
        \sum_{t=1}^T
            \rbrm[\Big]{
            \frac{1}{\eta_{t+1}}
            -
            \frac{1}{\eta_{t}}
            }
            \left(
            \psi(x^*)
            -
            \psi( x_{t+1} )
            \right)
        \nonumber
        \\
        &
        \quad
        \quad
        +
        \frac{1}{\eta_1} \left(
        \psi(x^*)
        -
        \psi(x_1)
        \right)
        -
        \frac{1}{\eta_{T+1}}
        D(x^*, x_{T+1}),
        \label{eq:bound_FTRL_with_negative}
    \end{align}
    where $D(\cdot, \cdot )$ is the Bregman divergence associated with $\psi$:
    $D(y, x) = \psi(y) - \psi(x) - \nabla \psi(x)^\top (y - x)$.
\end{lemma}
\begin{proof}
    We have
    \begin{align}
        &
        \sum_{t=1}^T \ell_t^\top x^* + \frac{1}{\eta_{T+1}} \psi(x^*)
        \nonumber \\
        &
        \ge
        \sum_{t=1}^T \ell_t^\top x_{T+1} + \frac{1}{\eta_{T+1}} \psi(x_{T+1})
        +
        \frac{1}{\eta_{T+1}} D(x^*, x_{T+1})
        \nonumber \\
        &
        =
        \sum_{t=1}^{T-1} \ell_t^\top x_{T+1} + 
        \frac{1}{\eta_{T}} \psi(x_{T+1})
        +
        \ell_{T}^\top x_{T+1}
        +
        \left(
        \frac{1}{\eta_{T+1}} 
        -
        \frac{1}{\eta_{T}} 
        \right)
        \psi(x_{T+1})
        +
        \frac{1}{\eta_{T+1}} D(x^*, x_{T+1})
        \nonumber \\
        &
        \ge
        \sum_{t=1}^{T-1} \ell_t^\top x_{T} + 
        \frac{1}{\eta_{T}} \psi(x_{T})
        +
        \frac{1}{\eta_{T}} D(x_{T+1},x_T)
        +
        \ell_{T}^\top x_{T+1}
        \nonumber\\
        &\qquad+
        \rbrm[\Big]{
        \frac{1}{\eta_{T+1}} 
        -
        \frac{1}{\eta_{T}} 
        }
        \psi(x_{T+1})
        +
        \frac{1}{\eta_{T+1}} D(x^*, x_{T+1})
        \nonumber \\
        &
        \ge
        \cdots
        \nonumber \\
        &\ge
        \frac{1}{\eta_1} \psi(x_{1})
        +
        \sum_{t=1}^{T} 
        \left(
        \frac{1}{\eta_t}
        D(x_{t+1}, x_t)
        +
        \ell_t^\top x_{t+1}
        +
        \rbrm[\Big]{
        \frac{1}{\eta_{t+1}} 
        -
        \frac{1}{\eta_{t}} 
        }
        \psi(x_{t+1})
        \right)
        +
        \frac{1}{\eta_{T+1}} D(x^*, x_{T+1}),
        \nonumber 
    \end{align}
    where each inequality follows from the definition of FTRL \eqref{eq:defFTRL}
    and the first-order optimality condition.
    This immediately leads to the desired inequality.
\end{proof}

We next provide lemmas to the upper bound the first summation in the RHS of \eqref{eq:bound_FTRL_with_negative} for the $1/2$-Tsallis entropy.
\begin{lemma}\label{lem:tsalis_stab_onedim}
Let $\phi(x) = - 2 \sqrt{x}$
and
$D_\phi(y, x) = - 2 \sqrt{y} + 2 \sqrt{x} + (y-x)/\sqrt{x} = (\sqrt{y} - \sqrt{x})^2 / \sqrt{x}$ 
be the Bregman divergence associated with $\phi$.
Then, for any 
$x \in (0,1)$
and 
$a > -1/\sqrt{x}$,
\begin{equation}
    \max_{y \in (0, \infty)} 
    \left\{
        a (x - y) - D_\phi(y, x) 
    \right\}
    \leq 
    \sqrt{x} \, \xi(a \sqrt{x})
    \nonumber
\end{equation}
for $\xi(z) = z^2 / (1 + z)$ for $z \geq 0$.
If $a \geq - 1/(2\sqrt{x})$, then it also holds that 
\begin{equation}
    \max_{y \in (0, \infty)} 
    \left\{
        a (x - y) - D_\phi(y, x) 
    \right\}
    \leq 
    2 x^{3/2} a^2
    .
    \nonumber
\end{equation}
\end{lemma}
\begin{proof}
We have
\begin{align}
  a (x - y) - D_\phi(y, x)
  &=
  a (x - y)
  -
  \frac{(\sqrt{y} - \sqrt{x})^2}{\sqrt{x}}
  \nonumber \\
  &=
  (\sqrt{x} - \sqrt{y})
  \left\{
      a (\sqrt{x} + \sqrt{y})
      -
      \frac{ \sqrt{x} - \sqrt{y} }{\sqrt{x}}
  \right\}
  \nonumber \\
  &=
  (\sqrt{x} - \sqrt{y})
  \left\{
      a \left(2 \sqrt{x} - (\sqrt{x} - \sqrt{y}) \right)
      -
      \frac{ \sqrt{x} - \sqrt{y} }{\sqrt{x}}
  \right\}
  \nonumber \\
  &=
  2 a \sqrt{x}
  (\sqrt{x} - \sqrt{y})
  -
  \left(
    a + \frac{1}{\sqrt{x}}
  \right)
  (\sqrt{x} - \sqrt{y})^2
  \nonumber \\
  &\leq
  \frac{(2 a \sqrt{x})^2}{4 \left(a + \frac{1}{\sqrt{x}}\right)}
  =
  \frac{a^2 x^{3/2}}{a \sqrt{x} + 1}
  =
  \sqrt{x} \xi(a \sqrt{x})
  ,
  \nonumber
\end{align}
where we used $c_1 z - c_2 z^2 \leq c_1^2 / (4 c_2)$ for $c_1 \geq 0$ and $c_2 > 0$ with $a > - 1/\sqrt{x}$.
The second statement of the lemma follows since 
$\xi(z) \leq 2 z^2$ for any $z \geq - 1/2$.
This completes the proof.
\end{proof}

Define 
\begin{equation}
    \tilde{\ell}_t(i)
    =
    \frac{\one[i_t=i](1 - r_t)}{x_t(i)} \in \left[0, \frac{2}{x_t(i)} \right]
    .
    \nonumber
\end{equation}
Then, using Lemma~\ref{lem:tsalis_stab_onedim}, we can prove the following lemma, which plays a key role in proving Theorem~\ref{thm:Tsallis-INF}.
\begin{lemma}\label{lem:tsalis_stab_multidim}
Suppose that $\eta_t \leq 1/4$.
Then it holds that
\begin{equation}
        \tilde{\ell}_t^\top 
    (x_t - x_{t+1})
    -
    \frac{1}{\eta_t} D(x_{t+1}, x_t)
    \leq
    4
    \eta_t
    \sum_{i=1}^m
    x_t(i)^{3/2} (1 - x_t(i))
    \tilde{\ell}_t(i)^2
    .
    \nonumber
\end{equation}
\end{lemma}
\begin{proof}
Let $k \in \argmax_{i \in [m]} x_t(i)$.
We then have 
\begin{align}
    \tilde{\ell}_t^\top (x_t - x_{t+1})
    -
    \frac{1}{\eta_t} D(x_{t+1}, x_t)
    &=
    \left(
        \tilde{\ell}_t - x_t(k) \tilde{\ell}_t(k) \mathbf{1}
    \right)^\top 
    (x_t - x_{t+1})
    -
    \frac{1}{\eta_t} D(x_{t+1}, x_t)
    \nonumber \\
    &\leq
    2
    \eta_t
    \sum_{i=1}^m 
    x_t(i)^{3/2} \left( 
        \tilde{\ell}_t(i) - x_t(k) \tilde{\ell}_t(k)
    \right)^2
    ,
    \label{eq:stab_0}
\end{align}
where in the inequality we used  
the second statement in Lemma~\ref{lem:tsalis_stab_onedim}
with
\begin{equation}
    \sqrt{x_t(i)} \eta_t 
    \left( \tilde{\ell}_t(i) -  x_t(k) \tilde{\ell}_t(k) \right)
    \geq
    - \sqrt{x_t(i)} \eta_t x_t(k) \tilde{\ell}_t(k)
    \geq
    - \eta_t (1 - r_t)
    \geq 
    - \frac12
    \nonumber
\end{equation}
for each $i \in [m]$,
which is due to the assumption that $\eta_t \leq 1/4$ and $1 - r_t \in [0, 2]$.
We will upper bound the RHS of~\eqref{eq:stab_0} below.
First we have
\begin{align}
    &
    \sum_{i=1}^m 
    x_t(i)^{3/2} \left( 
        \tilde{\ell}_t(i) - x_t(k) \tilde{\ell}_t(k)
    \right)^2
    \nonumber \\
    &=
    x_t(k)^{3/2}
    \left( 
        1 - x_t(k)
    \right)^2
    \tilde{\ell}_t(k)^2
    +
    \sum_{i \neq k}
    x_t(i)^{3/2} \left( 
        \tilde{\ell}_t(i) - x_t(k) \tilde{\ell}_t(k)
    \right)^2
    .
    \label{eq:stab_1}
\end{align}
The second term in the last equality is upper bounded by
\begin{align}
    \sum_{i \neq k}
    x_t(i)^{3/2} \left( 
        \tilde{\ell}_t(i) - x_t(k) \tilde{\ell}_t(k)
    \right)^2
    \leq
    \sum_{i \neq k}
    x_t(i)^{3/2} 
        \tilde{\ell}_t(i)^2 
    +
    x_t(k)^2 \tilde{\ell}_t(k)^2
    \sum_{i \neq k}
    x_t(i)^{3/2} 
    \label{eq:stab_2}
    .
\end{align}
The second term in the last inequality is further upper bounded by
\begin{align}
    x_t(k)^2 \tilde{\ell}_t(k)^2
    \sum_{i \neq k}
    x_t(i)^{3/2} 
    &\leq
    x_t(k)^2 \tilde{\ell}_t(k)^2
    \rbrm[\bigg]{
        \sum_{i \neq k}
        x_t(i)
    }^{3/2}
    \leq
    x_t(k)^{3/2} \tilde{\ell}_t(k)^2
    \rbrm[\bigg]{
        \sum_{i \neq k}
        x_t(i)
    }
    \nonumber \\
    &
    =
    x_t(k)^{3/2} \tilde{\ell}_t(k)^2
    \left(
        1 - x_t(k)
    \right)
    ,
    \label{eq:stab_3}
\end{align}
where the first inequality follows from the superadditivity of $z \mapsto z^{3/2}$ for $z \geq 0$
and the second inequality follows from $\sum_{i \neq k} x_t(i) \in [0,1]$.
Combining \eqref{eq:stab_1}, \eqref{eq:stab_2}, and \eqref{eq:stab_3},
we have
\begin{align}
    &
    \sum_{i=1}^m 
    x_t(i)^{3/2} \left( 
        \tilde{\ell}_t(i) - x_t(k) \tilde{\ell}_t(k)
    \right)^2    
    \nonumber \\
    &\leq
    x_t(k)^{3/2}
    \left( 
        1 - x_t(k)
    \right)^2
    \tilde{\ell}_t(k)^2
    +
    \sum_{i \neq k}
    x_t(i)^{3/2}
    \tilde{\ell}_t(i)^2
    +
    x_t(k)^{3/2}
    \tilde{\ell}_t(k)^2
    (1 - x_t(k))
    \nonumber \\
    &\leq
    2
    x_t(k)^{3/2}
    \left( 
        1 - x_t(k)
    \right)
    \tilde{\ell}_t(k)^2
    +
    2
    \sum_{i \neq k}
    x_t(i)^{3/2} (1 - x_t(i))
    \tilde{\ell}_t(i)^2
    \nonumber \\
    &=
    2
    \sum_{i=1}^m
    x_t(i)^{3/2} (1 - x_t(i))
    \tilde{\ell}_t(i)^2
    ,
    \nonumber
\end{align}
where the second inequality follows from $1 - x_t(i) \geq 1/2$ for $i \neq k$ since $x_t(i) \leq 1/2$ for $i \neq k$.
Finally, combining \eqref{eq:stab_0} with the last inequality, we obtain the desired bound.
\end{proof}

Finally, we are ready to prove Theorem \ref{thm:Tsallis-INF}.
We will show that Theorem~\ref{thm:Tsallis-INF} holds with $C_1 = 19$ and $C_2 = 2$.
\begin{proof}[Proof of Theorem \ref{thm:Tsallis-INF}]
Let $x \in \cP_m$ and $T_0 = 4$.
When $m = 1$, the LHS and RHS of~\eqref{eq:Tsallis-INF-upper} are $0$, and thus we consider the case of $m \geq 2$ below.
Recall 
$
\tilde{\ell}_t(i)
=
\frac{\one[i_t=i](1 - r_t)}{x_t(i)}.
$
Then, the regret of the row player can be rewritten as
\begin{align}
    &
    \Rr_T(x)
    =
    \E\sbrm[\bigg]{
        \sum_{t=1}^T (x_t - x)^\top \ell_t
    }
    \leq
    \E\sbrm[\bigg]{
        \sum_{t=T_0+1}^T (x_t - x)^\top \ell_t
    }
    +
    2 T_0
    =
    \E\sbrm[\bigg]{
        \sum_{t=T_0+1}^T (x_t - x)^\top \hat{\ell}_t
    }
    +
    2 T_0
    \nonumber \\
    &=
    \E\sbrm[\bigg]{
        \sum_{t=T_0+1}^T (x_t - x)^\top \tilde{\ell}_t
        +
        \sum_{t=T_0+1}^T (x_t - x)^\top \left( \hat{\ell}_t - \tilde{\ell}_t \right)
    }
    +
    2 T_0
    =
    \E\sbrm[\bigg]{
        \sum_{t=T_0+1}^T (x_t - x)^\top \tilde{\ell}_t
    }
    +
    2 T_0
    ,
    \label{eq:ftrl_tsallis_0}
\end{align}
where the second equality follows from the unbiasedness of $\hat{\ell}_t$
and the last equality follows from $\hat{\ell}_t - \tilde{\ell}_t = \mathbf{1}$.
From the fact that the outputs of FTRL with loss estimator $\hat{\ell}_t$ and $\tilde{\ell}_t$ are the same and Lemma~\ref{lem:ftrl_bound}, the inside of the expectation in \eqref{eq:ftrl_tsallis_0} is upper bounded by
\begin{align}
    \sum_{t=T_0+1}^T \tilde{\ell}_t^\top \left( x_t - x \right)
    &\leq
    \sum_{t=T_0+1}^T
    \rbrm[\Big]{
        \tilde{\ell}_t^\top (x_t - x_{t+1})
        -
        \frac{1}{\eta_t} D(x_{t+1}, x_t)
        }
    \nonumber \\
    &\qquad+
    \sum_{t=T_0+1}^T
    \rbrm[\Big]{
        \frac{1}{\eta_{t+1}}
        -
        \frac{1}{\eta_{t}}
    }
    \rbrm[\Big]{
        \psi(x^*)
        -
        \psi( x_{t+1} )
    }
    \nonumber \\
    &\qquad+
    \frac{1}{\eta_{T_0+1}} \left(
        \psi(x^*)
        -
        \psi(x_{T_0+1})
    \right)
    -
    \frac{1}{\eta_{T+1}}
    D(x^*, x_{T+1}),
    \label{eq:ftrl_tsallis_1}
\end{align}
We first consider the first term in \eqref{eq:ftrl_tsallis_1}.

For $t \geq T_0 = 4$, 
we have $\eta_t = 1/(2\sqrt{t}) \leq 1/4$,
and thus from Lemma~\ref{lem:tsalis_stab_multidim},
\begin{align}
    \tilde{\ell}_t^\top (x_t - x_{t+1})
    -
    \frac{1}{\eta_t} D(x_{t+1}, x_t)
    \leq
    4
    \eta_t
    \sum_{i=1}^m
    x_t(i)^{3/2} (1 - x_t(i))
    \tilde{\ell}_t(i)^2
    .
    \label{eq:ftrl_tsallis_stab2}
\end{align}
Let $i^* \in [m]$.
Then,
using $\E_{r_t,i_t,j_t}[ \tilde{\ell}_t(i)^2 \mid x_t] \leq 4 / x_t(i)$,
we have
\begin{align}
    &
    \E\nolimits_{r_t, i_t, j_t} \sbrm[\bigg]{
        \sum_{i=1}^m
        x_t(i)^{3/2} (1 - x_t(i))
        \tilde{\ell}_t(i)^2
        \mid
        x_t
    }
    \leq
    4
    \sum_{i=1}^m
    \sqrt{x_t(i)} (1 - x_t(i))
    \nonumber \\
    &\leq
    4
    \sum_{i \neq i^*}
    \sqrt{x_t(i)}
    +
    4
    (1 - x_t(i^*))
    \leq
    8
    \sum_{i \neq i^*}
    \sqrt{x_t(i)}
    ,
    \label{eq:ftrl_tsallis_stab_final}
\end{align}
where the last inequality follows from
$1 - x_t(i^*) = \sum_{i \neq i^*} x_t(i) \leq \sum_{i \neq i^*} \sqrt{x_t(i)}$.

We next consider the second and third terms in \eqref{eq:ftrl_tsallis_1}.
We first observe that 
$
{1}/{\eta_{t+1}}
-
{1}/{\eta_{t}}
=
2 (\sqrt{t+1} - \sqrt{t}) 
\leq
1/\sqrt{t}
\leq
\sqrt{2 / (t+1)}
$
and 
$\psi(x^*) - \psi(x_{t+1})
\leq
2 \sum_{i=1}^m \sqrt{x_{t+1}(i)} - 2
\leq
2 \sum_{i \in [m] \setminus \{i^*\}} \sqrt{x_{t+1}(i)}.
$
Using these inequalities, we have
\begin{align}
    &
    \sum_{t=T_0+1}^{T-1}
    \rbrm[\Big]{
        \frac{1}{\eta_{t+1}}
        -
        \frac{1}{\eta_{t}}
    }
    \left(
        \psi(x^*)
        -
        \psi( x_{t+1} )
    \right)
    \nonumber \\
    &\leq
    2 \sqrt{2} \sum_{t=T_0+1}^{T}
    \frac{1}{\sqrt{t+1}}
    \sum_{i \in [m] \setminus \{i^*\}} \sqrt{x_{t+1}(i)}
    \leq
    2 \sqrt{2} \sum_{t=T_0+2}^{T}
    \frac{1}{\sqrt{t}}
    \sum_{i \in [m] \setminus \{i^*\}} \sqrt{x_t(i)}
    +
    2 \sqrt{2} \sqrt{\frac{m}{T+1}}
    ,
    \label{eq:ftrl_tsallis_penalty_final}
\end{align}
where in the last inequality we used the Cauchy-Schwarz inequality.
The remaining term in \eqref{eq:ftrl_tsallis_1} is at most
\begin{align}
    \frac{1}{\eta_{T_0+1}} \left(
        \psi(x^*)
        -
        \psi(x_{T_0+1})
    \right)    
    \leq
    2 \sqrt{2} \sqrt{T_0 + 1} \sum_{i \neq i^* } \sqrt{x_{T_0+1}(i)}
    \leq
    2 \sqrt{10} \sum_{i \neq i^* } \sqrt{x_{T_0+1}(i)}
    .
    \label{eq:ftrl_tsallis_remaining}
\end{align}
Finally, by combining \eqref{eq:ftrl_tsallis_0} with \eqref{eq:ftrl_tsallis_1}, \eqref{eq:ftrl_tsallis_stab_final}, \eqref{eq:ftrl_tsallis_penalty_final}, and \eqref{eq:ftrl_tsallis_remaining},
we obtain that for any $i^* \in [m]$,
\begin{align}
    \Rr_T(x)
    &\leq
    2 T_0
    +
    2 \sqrt{2}
    \sqrt{\frac{m}{T+1}}
    +
    \E \sbrm[\bigg]{
        19
        \sum_{t=T_0+1}^T
        \sum_{i \in [m] \setminus \{i^*\}}
        \sqrt{\frac{x_t(i)}{t}}
        -
        2
        \sqrt{T+1}
        \cdot 
        D(x, x_{T+1})
    }.
    \label{eq:theorem1_pre}
\end{align}
Since we have 
$\sum_{t=1}^{T_0} \sum_{i \in [m] \setminus \{i^*\}} \sqrt{x_t(i)} 
\geq 
(\sqrt{m} - 1) + (T_0 - 1)
=
\sqrt{m} + 2
$,
the last inequality implies that the choice of 
$C_1 = 19$, which is larger than $\frac{2 T_0 + 2 \sqrt{2} \sqrt{m / (T+1)}}{\sqrt{m} + 2} (\leq 4)$,
implies that 
the first three terms in \eqref{eq:theorem1_pre} is upper bounded by
\begin{equation}
    2 T_0
    +
    2 \sqrt{2}
    \sqrt{\frac{m}{T+1}}
    \leq
    C_1 (\sqrt{m} + 2)
    \leq 
    C_1 
    \sum_{t=1}^{T_0}
    \sum_{i \in [m] \setminus \{i^*\}} \sqrt{x_t(i)}    
    .
    \nonumber
\end{equation}
Combining this inequality with \eqref{eq:theorem1_pre} implies that 
Theorem~\ref{thm:Tsallis-INF} holds with $C_1 = 19$ and $C_2 = 2$ as desired.
\end{proof}

\section{Properties of the Duality Gap}\label{sec:app:duality-gap}

\begin{lemma}
    \label{lem:DGLipschitz}
    For any $A \in [-1, 1]^{m \times n}$,
    $\DG(\hat{x},\hat{y})$ defined in \eqref{eq:DG}
    is $1$-Lipschitz w.r.t.~$L^1$ norm,
    i.e.,
    it holds for any 
    $\hat{x}, \hat{x}' \in \cP_m$ and
    $\hat{y}, \hat{y}' \in \cP_n$ that
    \begin{align}
        |
        \DG( \hat{x}, \hat{y} )
        -
        \DG( \hat{x}', \hat{y}' )
        |
        \le
        \| \hat{x} - \hat{x}' \|_1
        +
        \| \hat{y} - \hat{y}' \|_1.
    \end{align}
\end{lemma}
\begin{proof}
    We can express $\DG(\hat{x}, \hat{y})$ as follows:
    \begin{align}
        \DG( \hat{x}, \hat{y} )
        =
        \max_{x \in \cP_m} \cbrm[\big]{ x^\top A \hat{y} }
        +
        \max_{y \in \cP_n} \cbrm[\big]{ - \hat{x}^\top A {y} }.
        \label{eq:DGLi1}
    \end{align}
    Let $\tilde{x} \in \argmax_{ x \in \cP_m }\left\{ x^\top A \hat{y} \right\}$.
    We then have
    \begin{align}
        &
        \max_{x \in \cP_m} \cbrm[\big]{ x^\top A \hat{y} }
        -
        \max_{x \in \cP_m} \cbrm[\big]{ x^\top A \hat{y}' }
        =
        \tilde{x}^\top A \hat{y} 
        -
        \max_{x \in \cP_m} \cbrm[\big]{ x^\top A \hat{y}' }
        \le
        \tilde{x}^\top A \hat{y} 
        -
        \tilde{x}^\top A \hat{y}' 
        \\
        &
        =
        \tilde{x}^\top A (\hat{y} - \hat{y}')
        \le
        \| A^\top \tilde{x} \|_{\infty} 
        \| \hat{y} - \hat{y}' \|_1
        \le
        \| \hat{y} - \hat{y}' \|_1.
        \label{eq:DGLi2}
    \end{align}
    In a similar way,
    we can show the following:
    \begin{align}
        \max_{y \in \cP_n} \cbrm[\big]{ - \hat{x}^\top A {y} }
        -
        \max_{y \in \cP_n} \cbrm[\big]{ - \hat{x}'^\top A {y} }
        \le
        \| \hat{x} - \hat{x}' \|_1.
        \label{eq:DGLi3}
    \end{align}
    By comibning \eqref{eq:DGLi1}, \eqref{eq:DGLi2} and \eqref{eq:DGLi3},
    we obtain
    \begin{align}
        \DG( \hat{x}, \hat{y} )
        -
        \DG( \hat{x}', \hat{y}' )
        \le
        \| \hat{x} - \hat{x}' \|_1
        +
        \| \hat{y} - \hat{y}' \|_1.
    \end{align}
    In a similar way,
    we can show
    $
        \DG( \hat{x}', \hat{y}' )
        -
        \DG( \hat{x}, \hat{y} )
        \le
        \| \hat{x} - \hat{x}' \|_1
        +
        \| \hat{y} - \hat{y}' \|_1
    $
    as well,
    which completes the proof.
\end{proof}

\section{Proof of Theorem \ref{thm:general-bound-together}}\label{sec:app:proof-thm2}
In this appendix section, we will prove our main regret bound theorem by first presenting a generalized formulation that encompasses both inequalities stated in the main text. We do this by first defining a unified notation of the gap parameters.

\begin{definition}[Admissible $(I, \Deltar, \pir)$ and $(J, \Deltac, \pic)$]\label{def:admissble}
    Denote $v = \max_{x \in \cP_m}\left\{ \min_{y \in \cP_n} x^\top A y \right\}$.
    An action subset $I\subseteq [m]$, a gap vector $\Deltar\in \fR_{\geq 0}^m$ and a mapping $\pir : \cP_m\to \xstarset$ are together called {admissible} for the row player if
    \begin{itemize}[leftmargin=*]
    \item The entries $\Deltar(i)$ are positive for every $i\not\in I$. %
    \item For any $x\in \cP_m$, the NE strategy $\xstar=\pir(x)\in \xstarset$ must satisfy:
    \begin{equation}
        \DG(x, \ystar) = v - \min_{y\in \cP_n}\cbrm[\big]{x^\top A y} \geq \Deltar \cdot \rbrm{ x - \xstar }_+, \label{eq:def-Delta-r}
    \end{equation}
where $\ystar\in\ystarset$ is an arbitrary NE strategy, and we define $(x)_+ = \max\{ x, 0 \}$ which applies entrywise to vectors.

The admissibility for a subset of actions $J\subseteq [n]$, a gap vector $\Deltac\in \fR_{\geq 0}^n$, and a mapping $\pic : \cP_n\to\ystarset$ can be analogously defined for the column player, with
\begin{equation}
        \DG(\xstar, y) = \min_{x\in \cP_m}\cbrm[\big]{x^\top A y} - v  \geq \Deltac \cdot \rbrm{ y - \ystar }_+, \label{eq:def-Delta-c}
\end{equation}
\end{itemize}
\end{definition}

The following is the full version of our main theorem.

\begin{theorem}\label{thm:general-bound-together-appx}
If both players follow the Tsallis-INF algorithm, then for any admissible $(I, \Deltar, \pir)$ and $ (J, \Deltac, \pic)$ (Definition~\ref{def:admissble}) such that $I \neq \emptyset$, $J \neq \emptyset $, we have
    \begin{align*}
       & \max\cbrm[\Big]{
            \Rr_T(x)
            +
            \sqrt{T}\constd{\E\sbrm[\big]{
                D(x, x_{T+1})
            }},
            \Rc_T(y)
            +
            \sqrt{T}\constd{\E\sbrm[\big]{
                D(y, y_{T+1})
            }}
        }\\
        &{}={} 
        O\rbrm[\Big]{
            \sqrt{T} \rbrm[\big] {
                \sqrt{\abs{I}-1}
                + \sqrt{\abs{J}-1}
                + \gammar \sqrt{\Logr}
                + \gammac \sqrt{\Logc}
            }
            +
            \omegar \Logr
            +
            \omegac \Logc
        } \yestag\label{eq:general-bound-together}
    \end{align*}
    for any $x \in \cP_m$ and $y \in \cP_n$,
    where
\begin{equation}
    \begin{aligned}
    \omegar & = \sum_{i \not\in I}
        \frac{1}{\Deltar(i)}, & 
    \gammar & = \max_{\xstar \in \xstarset} \sum_{i \not\in I} \sqrt{\xstar(i)}, &
    \Logr & = \min\cbrm[\Big]{
            \log_+ \frac{
                    T\rbrm{m-\absm{I}}
                }{
                    \omegar^2
                },
            \log_+ \frac{
                    m-\absm{I}
                }{
                    \gammar^2
                }
        },
        \\
    \omegac & = \sum_{j \not\in J}
        \frac{1}{\Deltac(j)}, & 
    \gammac & = \max_{\ystar \in \ystarset} \sum_{j \not\in J} \sqrt{\ystar(j)}, &
     \Logc & = \min\cbrm[\Big]{
            \log_+ \frac{
                    T\rbrm{n-\absm{J}}
                }{
                    {\omegac}^2
                },
            \log_+ \frac{
                    n-\absm{J}
                }{
                    \gammac^2
                }
        }.
    \end{aligned}    
    \label{eq:def-omega}
\end{equation}
\end{theorem}

We note that the first part of Theorem~\ref{thm:general-bound-together} is a special cases of this theorem.
In fact,
if $(\xstar, \ystar, I,J,\Deltar,\Deltac)$ are given by the first condition in Theorem~\ref{thm:general-bound-together}, we can verify the admissibility of $(I,\Deltar)$ by directly plugging the definition $\Deltar=\rbrm[\big]{\xstar^\top A \ystar}\one - A \ystar$ into \eqref{eq:def-Delta-r}; and similarly admissibility can be proven for $(J, \Deltac)$.

The first bound in Theorem~\ref{thm:general-bound-together} can be recovered by observing that $\gammar=\gammac=0$ as we define $I$ and $J$ to be the support for $\xstar$ and $\ystar$, and by taking the second branch in the definitions of $\Logr$ and $\Logc$.

\begin{proof}[of Theorem~\ref{thm:general-bound-together-appx}]
From Equation~\eqref{eq:def-Delta-r} and \eqref{eq:def-Delta-c}, we have
\begin{align*}
    \Rr_T + \Rc_T
    & = \max_{ x \in \cP_m, y \in \cP_n } {
            \E \sbrm[\bigg]{
                \sum_{t=1}^T \rbrm[\big]{
                    x^\top A y_t
                    -
                    x_t^\top A y
                }
            }
        } \\
    & = \max_{x \in \cP_m} \cbrm[\bigg]{
            x^\top A
            \E \sbrm[\bigg]{
                \sum_{t=1}^T y_t
            }
        }
      - \min_{ y \in \cP_n } \cbrm[\bigg]{
            \E \sbrm[\bigg]{
                \sum_{t=1}^T x_t
            }
            A y
        } \\
    & = T \max_{x \in \cP_m} \cbrm[\bigg]{
            x^\top A
            \E [\ybar_T]
        }
      - T \min_{ y \in \cP_n } \cbrm[\bigg]{
            \E [\xbar_T]
            A y
        } \\
    & \ge
    T \Deltar \cdot \rbrm{ \E[\xbar_T] - \xstar }_+
    +
    T \Deltac \cdot \rbrm{ \E[\ybar_T] - \ystar }_+,
    \yestag\label{eq:general-regret-bounded-below}
\end{align*}
where we define $\xbar_T=\frac{1}{T}\sum_{t=1}^T x_t$ and 
$\ybar_T=\frac{1}{T}\sum_{t=1}^T y_t$.

From Theorem~\ref{thm:Tsallis-INF},
we know that the following bound holds:
\begin{align}
    &
    \Rr_T(x) +
    \constd\sqrt{T}\E\sbrm[\big]{
        D(x, x_{T+1})
    }
    \le
    \consti\E\sbrm[\bigg]{
        \sum_{t=1}^{T} {
            \frac{1}{\sqrt{t}}
            \sum_{i \neq \istar} {
                \sqrt{x_t(i)}
            }
        }
    }
    , \label{eq:general-regret-bounded-above-specific}
\end{align}
for an arbitrary $\istar\in I$,
and specifically
\begin{align}
    &
    \Rr_T
    \le
    \consti\E\sbrm[\bigg]{
        \sum_{t=1}^{T} {
            \frac{1}{\sqrt{t}}
            \sum_{i \neq \istar} {
                \sqrt{x_t(i)}
            }
        }
    }
    \defeq
    \consti \E[S]
    . \label{eq:general-regret-bounded-above}
\end{align}
Define $\Sr$ as the summation inside the expectation bracket above.
We can split it into two parts:
\begin{align}
    \Sr = 
    \sum_{t=1}^{T} {
        \frac{1}{\sqrt{t}}
        \sum_{i \in I \setminus \cbr{\istar}} {
            \sqrt{x_t(i)}
        }
    }
    +
    \sum_{t=1}^{T} {
        \frac{1}{\sqrt{t}}
        \sum_{i \not\in I} {
            \sqrt{x_t(i)}
        }
    }. \label{eq:general-i-bound-by-I}
\end{align}
The sum within $I$ can be bounded with a Cauchy-Schwarz inequality:
\begin{align*}
    \sum_{t=1}^{T} {
        \frac{1}{\sqrt{t}}
        \sum_{i \in I \setminus \cbr{\istar}} {
            \sqrt{1 \cdot x_t(i)}
        }
    } & \leq \sum_{t=1}^{T} {
        \frac{1}{\sqrt{t}}
        \sqrt{
            \sum_{i \in I \setminus \cbr{\istar}} {
                1
            }
            \sum_{i \in I \setminus \cbr{\istar}} {
                \sqrt{x_t(i)}^2
            }
        }
    } \\
    & \leq
    2\sqrt{T}\sqrt{\abs{I}-1}. \yestag\label{eq:general-i-within-I}
\end{align*}
We define $\xbar(i)=\frac{1}{T}\sum_{t=1}^T x_t(i)$ as a notational shorthand.
To handle the sum outside $I$, we have due to the Cauchy-Schwarz inequality,
\begin{align*}
    \sum_{t=1}^{T} {
        \frac{1}{\sqrt{t}}
        \sum_{i \not\in I} {
            \sqrt{x_t(i)}
        }
    } & 
    = \sum_{t=1}^{s} {
        \frac{1}{\sqrt{t}}
        \sum_{i \not\in I} {
            \sqrt{x_t(i)}
        }
    } + \sum_{i \not\in I} {
        \sum_{t=s+1}^{T} {
            \frac{1}{\sqrt{t}}
            \sqrt{x_t(i)}
        }
    } \\
    & \leq
    \sum_{t=1}^{s} {
        \frac{1}{\sqrt{t}}
        \sqrt{m-\absm{I}}
    }
    +
    \sum_{i \not\in I} \sqrt{
        \rbrm[\Big]{ \sum\nolimits_{t=s+1}^{T} {
            1/t
        } }
        \rbrm[\Big]{ \sum\nolimits_{t=s+1}^{T} {
            x_t(i)
        } }
    } \\
    &\leq 2 \sqrt{s \rbrm{m-\absm{I}}}
    +
    \sqrt{T \log \rbrm{T/s}}
    \sum_{i \not\in I} \sqrt{
        \xbar_T(i)
    },
    \yestag\label{eq:ng-outside-expand-i}
\end{align*}
in which $s\in[T]$ is a parameter yet to be determined. 
We bound the last summation above in expectation. 
Definition~\ref{def:admissble} guarantees that, for $\E[\xbar_T]$, it is possible to find a Nash equilibrium strategy $\xstar$ such that the following holds:
\begin{align*}
    \E\sbrm[\Big]{\sum_{i \not\in I} \sqrt{
        \xbar_T(i)
    }} & \leq
    \sum_{i \not\in I} \sqrt{
        \E[\xbar_T(i)]
    }
    \\
    & \leq
    \sum_{i \not\in I} \sqrt{
        \xstar(i)
    }
    +
    \sum_{i \not\in I} \sqrt{
        \rbrm{\E[\xbar_T(i)] - \xstar(i)}_{+}
    }\\
    & \leq
    \gamma
    +
    \sum_{i \not\in I} \sqrt{
        \frac{1}{\Deltar(i)}
        \rbrm[\Big]{
            \Deltar(i) \cdot
            \rbrm{\E[\xbar_T(i)] - \xstar(i)}_{+}
        }
    } \\
    & \leq
    \gamma
    +
    \sqrt{
        \rbrm[\Big]{
            \sum_{i \not\in I}
            \frac{1}{\Deltar(i)}
        }
        \rbrm[\Big]{
            \sum_{i \not\in I}
            \Deltar(i) \cdot
            \rbrm{\E[\xbar_T(i)] - \xstar(i)}_{+}
        }
    } \\
    & = 
    \gamma + \sqrt{\omega \Deltar \cdot \rbrm{ \E[\xbar_T] - \xstar }_+}.
\end{align*}
We put \eqref{eq:general-i-bound-by-I}, \eqref{eq:general-i-within-I}, \eqref{eq:ng-outside-expand-i} together, and take the expectation on both sides to get
\begin{align*}
    \E[\Sr]
    & \leq 
    2 \sqrt{T}\sqrt{\abs{I}-1}
    +
    2 \sqrt{\rbrm{m-\abs{I}}s} 
    + \sqrt{T \log \rbrm{T/s}} \rbrm[\big]{
        \gamma + 
        \sqrt{\omegar \Deltar \cdot 
              \rbrm{ \E\sbrm{\xbar} - \xstar }_+}
    }, \yestag\label{eq:general-bound-together-expand-i}
\end{align*}
where the last inequality is due to Jensen's inequality and the concavity of square root.
For similarly defined 
$\Sc$ and $s'$, we also have a similar bound:
\begin{align*}
    \E[\Sc]
    & \leq 
    2 \sqrt{T}\sqrt{\abs{J}-1}
    +
    2 \sqrt{\rbrm{n-\abs{J}}s'} 
    + \sqrt{T \log \rbrm{T/s'}} \rbrm[\big]{
        \gamma + 
        \sqrt{\omegac \Deltac \cdot 
              \rbrm{ \E\sbrm{\ybar} - \ystar }_+}
    } .
    \yestag\label{eq:general-bound-together-expand-j}
\end{align*}
Now, note that from Jensen's inequality, we have
\begin{align*}
    &
    \sqrt{T \log \rbrm{T/s} \omegar}
    \sqrt{ \Deltar \cdot 
          \rbrm{ \E\sbrm{\xbar_T} - \xstar }_+}
    +
    \sqrt{T \log \rbrm{T/s'} \omegac}
    \sqrt{ \Deltac \cdot 
          \rbrm{ \E\sbrm{\ybar_T} - \ystar }_+} \\
    & \leq
    \sqrt{
        T \rbrm[\big]{
          \log \rbrm{T/s} \omegar 
        + \log \rbrm{T/s'} \omegac
        } \rbrm[\big] {
            \Deltar \cdot 
            \rbrm{ \E\sbrm{\xbar_T} - \xstar }_+
            +
            \Deltac \cdot 
            \rbrm{ \E\sbrm{\ybar_T} - \ystar }_+
        }
    }\\
    & \leq
    \sqrt{
        \rbrm[\big]{
          \omegar \log \rbrm{T/s} 
        + \omegac \log \rbrm{T/s'} 
        } \rbrm[\big] {
            \Rr_T + \Rc_T
        }
    } \tag*{\textlangle\,from \eqref{eq:general-regret-bounded-below}\,\textrangle}\\
    & \leq
    \sqrt{
        \consti \rbrm[\big]{
          \omegar \log \rbrm{T/s} 
        + \omegac \log \rbrm{T/s'} 
        } \rbrm[\big] {
            \E[\Sr]+\E[\Sc]
        }
    } \tag*{\textlangle\,from \eqref{eq:general-regret-bounded-above} and its counterpart\,\textrangle}
\end{align*}
If we add \eqref{eq:general-bound-together-expand-i} and \eqref{eq:general-bound-together-expand-j}, we get the following bound:
\begin{align*}
    \E[\Sr+\Sc]
    \leq {} &
    2\sqrt{T}\rbrm[\big]{\sqrt{\abs{I}-1} + \sqrt{\abs{J}-1}}
    + 2 \sqrt{\rbrm{m-\abs{I}}s} 
    + 2 \sqrt{\rbrm{n-\abs{J}}s'}\\
    &
    {} +
    \gamma\sqrt{T}
    \sqrt{
        \log \rbrm{T/s}
    }
    +
    \gamma\sqrt{T}
    \sqrt{
        \log \rbrm{T/s'}
    } \\
    &
    {} +
    \sqrt{
        \consti \rbrm[\big]{
            \omegar \log \rbrm{T/s}
            +
            \omegac \log \rbrm{T/s'}
        }
        \rbrm{
            \E[S]+\E[S']
        }
    }\\
    \leq{} &
    4\sqrt{T}\rbrm[\big]{\sqrt{\abs{I}-1} + \sqrt{\abs{J}-1}}
    + 4 \sqrt{\rbrm{m-\abs{I}}s} 
    + 4 \sqrt{\rbrm{n-\abs{J}}s'}\\
    &
    {} +
    2\gamma\sqrt{T}
    \sqrt{
        \log \rbrm{T/s}
    }
    +
    2\gamma\sqrt{T}
    \sqrt{
        \log \rbrm{T/s'}
    } \\
    &
    +
        2 \consti
        \omegar \log \rbrm{T/s}
    +
        2 \consti
        \omegac \log \rbrm{T/s'}
    , \yestag\label{eq:general-bound-together-expand}
\end{align*}
where in the last inequality we apply Lemma~\ref{lem:sumszt-2}.
We take $s=\ceilm[\big]{
    \min\cbrm[\big]{
        \frac{T}{2},
        \frac{
            \max\cbrm{\omega^2,\, \gamma^2 T}
        }{
            m-\absm{I}
        }
    }
}$; since $\Deltar_i\leq 2$ for every $i$, we know that $\omegar \geq \frac{1}{2} \rbrm{m-\absm{I}}$, 
so we have $\frac{\omegar^2}{m-\absm{I}}\geq \frac{1}{4}$,
and thus the rounding-up increases $s$ by a factor of at most 4. This implies that $4 \sqrt{\rbrm{m-\abs{I}}s} \leq 16 \gamma \sqrt{T}+16\omegar$.

We also have
\begin{align*}
s & \geq 
    \min\cbrm[\Big]{
        \frac{T}{2},
        \frac{
            \max\cbrm{\omega^2, \gamma^2 T}
        }{
            m-\absm{I}
        }
    }, \\
\frac{T}{s} & \leq
    \max\cbrm[\Big]{
        2,
        \min\cbrm[\Big]{
        \frac{
            T\rbrm{m-\absm{I}}
        }{
            \omega^2
        },
        \frac{
            m-\absm{I}
        }{
            \gamma^2
        }
    }
}, \\
\log \frac{T}{s} & \leq 
\min\cbrm[\Big]{
    \log_+ \frac{
            T\rbrm{m-\absm{I}}
        }{
            \omega^2
        },
    \log_+ \frac{
            m-\absm{I}
        }{
            \gamma^2
        }
} \defeq L.
\end{align*}
A similar definition and respective inequalities are omitted for $s'$. Plugging these bounds into \eqref{eq:general-bound-together-expand} yields
\begin{align*}
    \E[\Sr+\Sc]
    \leq {} & 
    4\sqrt{T\rbrm{\absm{I}-1}} +
    18\gammar\sqrt{
        T \Logr
    }
    +
        (2\consti+16)
        \omegar \Logr
    \\
    &
    {} 
    +
    4\sqrt{T\rbrm{\absm{J}-1}}
    +
    18 \gammac \sqrt{
        T \Logc
    } 
    +
        (2\consti+16)
        \omegac \Logc
    .
\end{align*}
Together with \eqref{eq:general-regret-bounded-above} and the definition of $S$ and $S'$, this completes the proof.
\end{proof}

The second part of Theorem~\ref{thm:general-bound-together}
is a corollary of the following theorem:
\begin{theorem}
	\label{thm:general-bound3}
    Suppose that $c, c'\in (0, 1]$ satisfy 
    \begin{align}
        \label{eq:LBDG}
        \DG({x}, \hat{y})
        \ge
        c \min_{\xstar \in \NEr}\| {x} - \xstar \|_1
        +
        c' \min_{\ystar \in \NEc} \| {y} - \ystar   \|_1
    \end{align}
    for all $x\in \cP_m$ and $y \in \cP_n$.
    Define
    $\gamma \ge 0, \gamma' \ge 0 , \rhor > 0$ and
    $\rhoc > 0$ by
    \begin{align}
        \label{eq:defgamma}
        \gamma
        &
        =
        \max_{x \in \NEr}
        \cbrm[\bigg]{
        \sum_{i = 1}^m
        \sqrt{x(i)}
        }
        -1
        ,
        \quad
        \gamma'
        =
        \max_{y \in \NEc}
        \cbrm[\bigg]{
        \sum_{j = 1}^n
        \sqrt{y(j)}
        }
        -1,
        \\
        \rhor 
        &
        = 
        \gamma\sqrt{T \log_+ \rbrm[\bigg]{ \frac{m-1}{\gamma^2} }} 
        +
        \frac{m-1}{c}
        \log_+\rbrm[\bigg]{
            \frac{c^2 T}{m-1}
        },
        \\
        \rhoc 
        &
        = 
        \gamma'\sqrt{T \log_+ \rbrm[\bigg]{\frac{n-1}{\gamma'^2} }} 
        +
        \frac{n-1}{c'}
        \log_+\rbrm[\bigg]{
            \frac{c'^2 T}{n-1}
        }.
    \end{align}
    If both players follow the Tsallis-INF algoirthm,
    we have
    \begin{align*}
    \Rr_T(x)
    +
    \sqrt{T}\constd{\E\sbrm{
        D(x, x_{T+1})
    }}
    &= O\mleft(\rhor + \sqrt{(\rhor + \rhoc)\frac{m-1}{c} \log_+ \rbrm[\bigg]{\frac{c^2T}{m-1} }}\mright),
    \\
    \Rc_T(y)
    +
    \sqrt{T}\constd{\E\sbrm{
        D(y, y_{T+1})
    }}
    &= O\mleft(\rhoc + \sqrt{(\rhor + \rhoc)\frac{n-1}{c'} \log_+ \rbrm[\bigg]{\frac{c'^2T}{n-1} }}\mright)
    \end{align*}
    for any $x \in \cP_m$ and $y \in \cP_n$.
    Consequently,
    we have
    \begin{align*}
        \limsup_{T \to \infty}
        \frac{\Rr_T}{\sqrt{T}}
        =
        O\mleft(
        \gamma\sqrt{\log_+ \mleft( \frac{m-1}{\gamma^2} \mright)} 
        \mright),
        \quad
        \limsup_{T \to \infty}
        \frac{\Rc_T}{\sqrt{T}}
        =
        O\mleft(
        \gamma'\sqrt{\log_+ \mleft( \frac{n-1}{\gamma'^2} \mright)} 
        \mright).
    \end{align*}
\end{theorem}
From this theorem and the AM-GM inequality,
we have
\begin{align*}
    \max\mleft\{
        \Rr_T,
        \Rc_T
    \mright\}
    &
    =
    O\mleft(
        \rhor
        +
        \rhoc + \sqrt{(\rhor + \rhoc)\frac{n-1}{c'} 
        \mleft( 
            \log_+ \mleft(\frac{c^2T}{m-1} \mright)
            +
            \log_+ \mleft(\frac{c'^2T}{n-1} \mright)
        \mright)
        }
    \mright)
    \\
    &
    =
    O\mleft(
        \rhor + \rhoc
    \mright),
\end{align*}
which implies that the second part of Theorem~\ref{thm:general-bound-together} holds.
\begin{proof}
From Theorem~\ref{thm:Tsallis-INF},
for any $s \in [T]$ any $i^* \in [m]$,
and any $x \in \cP_m$,
we have
\begin{align}
    &
    \nonumber
    \Rr_T(x)
    +
    \sqrt{T}\constd{\E\sbrm{
        D(x, x_{T+1})
    }}
    \\
    &
    \nonumber
    =
    O\rbr{
        \E\sbrm[\bigg]{
        \sum_{t=1}^T
        \frac{1}{\sqrt{t}}
        \sum_{i \in [m] \setminus \{ i^* \}}
        \sqrt{x_t(i)}
        }
    }
    \\
    &
    \nonumber
    =
    O\mleft(
        \E\sbrm[\bigg]{
        \sum_{t=1}^s
        \frac{1}{\sqrt{t}}
        \sum_{i \in [m] \setminus \{ i^* \}}
        \sqrt{x_t(i)}
        +
        \sum_{t=s+1}^T
        \frac{1}{\sqrt{t}}
        \sum_{i \in [m] \setminus \{ i^* \}}
        \sqrt{x_t(i)}
        }
    \mright)
    \\
    &
    \nonumber
    =
    O\mleft(
        \sqrt{(m-1)s}
        +
        \sum_{i \in [m] \setminus \{ i^* \}}
        \sqrt{
            \E\sbrm[\bigg]{
            \sum_{t=s+1}^T
            x_t(i)
            }
            \log \frac{T}{s}
        }
    \mright)
    \\
    &
    =
    O\mleft(
        \sqrt{(m-1)s}
        +
        \sqrt{
        T
        \log \frac{T}{s}
        }
        \sum_{i \in [m] \setminus \{ i^* \}}
        \sqrt{
            \E \left[
            \bar{x}_T(i)
            \right]
        }
    \mright).
    \label{eq:Rrmx}
\end{align}
Denote
\begin{align}
    \label{eq:deftildex}
    \tilde{x}_T
    \in
    \argmin_{\xstar \in \NEr}\| \E [ \bar{x}_T] - \xstar \|_1,
    \quad
    \tilde{y}_T
    \in
    \argmin_{\ystar \in \NEc}\| \E[\bar{y}_T] - \ystar \|_1.
\end{align}
We then have
\begin{align}
    \nonumber
    &
    \sum_{i \in [m] \setminus \{ i^* \}}
    \sqrt{ \E[ \bar{x}_T(i) ]  }
    \\
    &
    \nonumber
    \le
    \sum_{i \in [m] \setminus \{ i^* \}}
    \sqrt{ \tilde{x}_T(i)} + 
    \sum_{i \in [m] \setminus \{ i^* \}}
    \sqrt{ | \E[ \bar{x}_T(i) ] - \tilde{x}_T(i)|  }
    \\
    &
    \nonumber
    \le
    \sum_{i \in [m] \setminus \{ i^* \}}
    \sqrt{ \tilde{x}_T(i)} 
    +
    \sqrt{
        (m-1)
        \sum_{i \in [m] \setminus \{ i^* \} }
    \abs{ \E[ \bar{x}_T(i) ] - \tilde{x}_T(i) }  }
    &
    (\mbox{Cauchy-Schwarz})
    \\
    &
    \nonumber
    \le
    \frac{1}{2}
    \left(
    \sum_{i =1}^m
    \sqrt{ \tilde{x}_T(i)} 
    -
    1
    \right)
    +
    \sqrt{
        (m-1)
        \nbr{ \E[ \bar{x}_T ] - \tilde{x}_T }_1
    }
    \\
    &
    \nonumber
    \le
    \frac{1}{2}
    \gamma
    +
    \sqrt{
        (m-1)
        \nbr{ \E[ \bar{x}_T ] - \tilde{x}_T }_1
    }
    &(\mbox{From \eqref{eq:defgamma} and \eqref{eq:deftildex}})
    \\
    &
    \le
    \frac{1}{2}
    \gamma
    +
    \sqrt{
        \frac{m-1}{c} \DG(\E[\bar{x}_T], \E[\bar{y}_T])
    }.
    &(\mbox{From \eqref{eq:LBDG} and \eqref{eq:deftildex}})
    \label{eq:sqrtxDG}
\end{align}
The third inequality can be shown 
by setting $i^* \in \argmax_{i \in [m]} \left\{ \tilde{x}_T(i) \right\}$.
From \eqref{eq:Rrmx} and \eqref{eq:sqrtxDG},
we have
\begin{align}
    \nonumber
    &
    \Rr_T(x)
    +
    \sqrt{T}\constd{\E\sbrm{
        D(x, x_{T+1})
    }}
    \\
    &
    =
    O\mleft(
        \sqrt{(m-1)s}
        +
        \sqrt{
        T
        \log \frac{T}{s}
        }
        \left(
        \gamma
        +
        \sqrt{
            \frac{m-1}{c} \DG(\E[\bar{x}_T], \E[\bar{y}_T])
        }
        \right)
    \mright).
    \label{eq:RrDG}
\end{align}
Similarly,
for any $s' \in [T]$,
we have
\begin{align*}
    &
    \Rc_T(y)
    +
    \sqrt{T}\constd{\E\sbrm{
        D(y, y_{T+1})
    }}
    \\
    &
    =
    O\mleft(
        \sqrt{(n-1)s'}
        +
        \sqrt{
        T
        \log \frac{T}{s'}
        }
        \left(
        \gamma'
        +
        \sqrt{
            \frac{n-1}{c'} \DG(\E[\bar{x}_T], \E[ \bar{y}_T])
        }
        \right)
    \mright).
\end{align*}
Here,
as we have
\begin{align*}
    T \cdot
    \DG( \E[\bar{x}_T],\E[ \bar{y}_T] )
    &
    =
    \Rr_T
    +
    \Rc_T,
\end{align*}
the value of
$ \DG( \E[\bar{x}_T],\E[ \bar{y}_T] ) $
is bounded as
\begin{align*}
    T \cdot
    \DG( \E[\bar{x}_T ],\E[ \bar{y}_T] )
    &
    =
    O\mleft(
    \sqrt{(m-1)s}
    +
    \sqrt{(n-1)s'}
    +
    \gamma
    \sqrt{
    T
    \log \frac{T}{s}
    }
    +
    \gamma'
    \sqrt{
    T
    \log \frac{T}{s'}
    }
    \mright.
    \\
    &
    \quad
    \quad
    \quad
    \left.
    +
    \sqrt{
        \left(
        \frac{m-1}{c} \log\frac{T}{s}
        +
        \frac{n-1}{c'} \log\frac{T}{s'}
        \right)
        T
        \cdot
        \DG(\E[\bar{x}_T], \E[ \bar{y}_T])
    }
    \right)
\end{align*}
for any $s, s' \in [T]$,
which implies
\begin{align*}
    &
    T \cdot
    \DG( \E[\bar{x}_T],\E[ \bar{y}_T] )
    \\
    &
    =
    O\mleft(
    \sqrt{(m-1)s}
    +
    \sqrt{(n-1)s'}
    +
    \gamma
    \sqrt{
    T
    \log \frac{T}{s}
    }
    +
    \gamma'
    \sqrt{
    T
    \log \frac{T}{s'}
    }
    +
        \frac{m-1}{c} \log\frac{T}{s}
        +
        \frac{n-1}{c'} \log\frac{T}{s'}
    \mright).
\end{align*}
By choosing 
\begin{align}
    \label{eq:tunesvalue}
    s = 
    \left\lceil
    \min\left\{ T,
        \max \left\{ 
            \frac{\gamma^2 T}{m-1},
            \frac{m-1}{c^2}
        \right\}
        \right\}
    \right\rceil
    \quad
    s' = 
    \left\lceil
    \min\left\{ T,
        \max \left\{ 
            \frac{\gamma'^2 T}{n-1},
            \frac{n-1}{c'^2}
        \right\}
        \right\}
    \right\rceil
\end{align}
we have
\begin{align*}
    &
    T \cdot
    \DG(\E[ \bar{x}_T], \E[ \bar{y}_T ] )
    \\
    &
    =
    O\mleft(
        \gamma\sqrt{T \log_+ \left( \frac{m-1}{\gamma^2} \right)} 
        +
        \frac{m-1}{c}
        \log_+\left( 
            \frac{c^2 T}{m-1}
        \right)
        +
        \gamma'\sqrt{T \log_+ \left( \frac{n-1}{\gamma'^2} \right)} 
        +
        \frac{n-1}{c'}
        \log_+\left( 
            \frac{c'^2 T}{n-1}
        \right)
    \mright)
    \\
    &
    =
    O\mleft(
        \rhor
        +
        \rhoc
    \mright).
\end{align*}
From this and \eqref{eq:RrDG}
with \eqref{eq:tunesvalue},
we have
\begin{align*}
    \Rr_T(x)
    +
    \sqrt{T}\constd{\E\sbrm{
        D(x, x_{T+1})
    }}
    &
    =
    O\mleft(
        \rhor
        +
        \sqrt{
            (\rhor + \rhoc)
            \frac{m-1}{c}
            \log\frac{T}{s}
        }
    \mright)
    \\
    &
    =
    O\mleft(
        \rhor
        +
        \sqrt{
            (\rhor + \rhoc)
            \frac{m-1}{c}
            \log_+\left( 
                \frac{c^2 T}{m-1}
            \right)
        }
    \mright).
\end{align*}
Similarly,
we obtain the desired upper bound on $
    \Rc_T(y)
    +
    \sqrt{T}\constd{\E\sbrm{
        D(y, y_{T+1})
    }}$,
which completes the proof.
\end{proof}

\section{Proof of Theorem~\ref{thm:RegLB}}
\label{sec:pfRegLB}
When $A$ is given by \eqref{eq:defADelta},
we can see that
$(\istar, \jstar)$ is a Nash equilibrium of the game with payoff matrix $A$.
In fact,
if $\xstar$ and $\ystar$ are the indicator vectors of $\istar$ and $\jstar$,
it holds for any $x \in \cP_m$ and $y \in \cP_n$ that
\begin{align*}
    x^\top A \ystar
    -
    \xstar^\top A y
    &
    =
    {\Deltac}^\top \ystar
    -
    x^\top \Deltar
    -
    {\Deltac}^\top y
    +
    \xstar^\top \Deltar
    =
    (\xstar - x)^\top \Deltar
    +
    (\ystar - y)^\top \Deltac
    \\
    &
    =
    - x^\top \Deltar
    - y^\top \Deltac
    =
    - 
    \rbrm[\Big]{
        \sum_{i \in [m]} \Deltar(i) x(i)
        +
        \sum_{j \in [n]} \Deltac(j) y(j)
    }
    \le
    0,
    \yestag\label{eq:NE}
\end{align*}
which means that $\DG(\xstar, \ystar)=0$.

In this section,
let $x_t \in \cP_m$ and $y_t \in \cP_n$ denote indicator vectors of $i_t \in [m]$ and $j_t \in [n]$,
respectively.
For any fixed algorithm and the true payoff matrix $A$,
we denote the regret of the algorithm as
\begin{align*}
    R_T (A)
    =
    \Rr_T(\xstar)
    +
    \Rc_T(\ystar)
    =
    \E \sbrm[\bigg]{
    \sum_{t=1}^T
    \rbrm[\big]{
    \xstar^\top A y_t
    -
    x_t^\top A \ystar
    }
    } .
\end{align*}
Then,
if $A$ is given by \eqref{eq:defADelta},
from \eqref{eq:NE},
we have
\begin{align}
    R_T(A)
    =
    \E \sbrm[\bigg]{
        \sum_{t=1}^T
        \rbrm[\big]{
            x_t^\top \Deltar
            +
            y_t^\top \Deltac
        }
    }
    =
    \sum_{i=1}^m
    \Deltar(i) \Tr_{T,i}(A)
    +
    \sum_{j=1}^n
    \Deltac(j) \Tc_{T,j}(A).
    \label{eq:RDeltaLB}
\end{align}

We can show Theorem~\ref{thm:RegLB} by using the following lemma:
\begin{lemma}
    \label{lem:TiLB}
    Suppose $A$ is given by \eqref{eq:defADelta}.
    Fix an arbitrary $i \in [m] \setminus \{ \istar \}$.
    Let 
    $\tilde{\Deltar} = \Deltar - 2 \Deltar_i \chi_i$
    and
    $
    \tilde{A}
    =
    \mathbf{1}_m {\Deltac}^\top 
    -
    \tilde{\Deltar} \mathbf{1}_n^\top 
    $.
    We then have
    \begin{align*}
        ({\Deltar}(i))^2
        \Tr_{T,i}(A)
        \ge
        \frac{1}{5}
        \ln
        {
        \frac{T}{2 \rbrm[\big]{
        \Tr_{T,i}(A) +  T - \Tr_{T,i}(\tilde{A})
        }}
        }.
    \end{align*}
\end{lemma}
\begin{proof}
    Note first that,
    for $p \in [3/8, 1/2]$ and $\delta \in [0, 1/4]$,
    we have
    \begin{align*}
        \KL( p , p + \delta )
        &
        =
        p \ln \frac{p}{p+\delta}
        +
        (1 - p ) \ln \frac{ 1-p}{1-p-\delta}
        \\
        &
        =
        -
        p \ln \rbrm[\Big]{
        1
        +
        \frac{\delta}{p}
        }
        -
        (1 - p ) \ln \rbrm[\Big]{
        1
        -
        \frac{\delta}{1-p}
        }
        \\
        &
        \le
        p \ln \rbrm[\Big]{
        -
        \frac{\delta}{p}
        +
        \rbrm[\big]{\frac{\delta}{p}}^2
        }
        +
        (1 - p ) \ln \rbrm[\Big]{
        \frac{\delta}{1-p}
        +
        \rbrm[\big]{
        \frac{\delta}{1-p}
        }^2
        }
        \\
        &
        =
        \frac{\delta^2}{p(1-p)}
        \le
        5 \delta^2.
        \yestag\label{eq:KLpD}
    \end{align*}
    Let $P$ and $\tilde{P}$ be distributions of
    $\cbrm{ (i_t, j_t, \ell_t) }_{t\in [T]}$ for $A$ and $\tilde{A}$,
    respectively.
    Then,
    from the Bretagnolle-Huber inequality (e.g., \citealp[Corollary 4]{canonne2022short}),
    we have
    \begin{align*}
        \TV ( P, \tilde{P} )
        \le
        1 - \frac{1}{2} \exp ( -\KL (P, \tilde{P}) ) .
    \end{align*}
    From the chain rule for the KL divergence (e.g., \citealp[Lemma 15.1]{lattimore2020bandit}),
    we have
    \begin{align*}
        \KL ( P, \tilde{P} )
        &
        =
        \E_{\{(i_t, j_t, \ell_t)\} \sim P} 
        \sbrm[\bigg]{
        \sum_{t=1}^T 
        \KL ( \Berpm( A_{i_t, j_t} ), \Berpm( \tilde{A}_{i_t, j_t} ) ) 
        }
        \\
        &
        \le
        \E_{\{(i_t, j_t, \ell_t)\} \sim P} 
        \sbrm[\bigg]{ 
        \sum_{t=1}^T 
        \mathbf{1}[i_t = i]
        \cdot
        5 (\Deltar(i))^2
        }
        =
        5 \Tr_{T, i} (A) (\Deltar(i))^2 ,
    \end{align*}
    where the inequality follows from the definition of $\tilde{A}$ and \eqref{eq:KLpD}.
    By combining above inequalities,
    we obtain
    \begin{align*}
        \frac{1}{T}|\Tr_{T, i}(A) - \Tr_{T, i}(\tilde{A})|
        \le
        \TV ( P, \tilde{P} )
        \le
        1 - \frac{1}{2} \exp ( -\KL (P, \tilde{P}) )
        \le
        1 - \frac{1}{2} \exp ( -
        5 \Tr_{T, i} (A) (\Deltar(i))^2 
        ),
    \end{align*}
    which implies that
    \begin{align*}
        \Tr_{T, i} (A) (\Deltar(i))^2 
        \ge
        \frac{1}{5}
        \ln 
        {
        \frac{T}{2 \rbrm[\big]{
        \Tr_{T, i}(A) +  T - \Tr_{T, i}(\tilde{A})
        }}
        }.
    \end{align*}
\end{proof}

\begin{proof}[of Theorem~\ref{thm:RegLB}]
    If $\tilde{A}$ is given as in Lemma~\ref{lem:TiLB},
    from \eqref{eq:RDeltaLB},
    we have
    \begin{align*}
        R_T(A)
        \ge
        \Deltar(i) \Tr_{T, i}(A),
        \quad
        R_T(\tilde{A})
        \ge
        \Deltar(i) (T - \Tr_{T,i}(\tilde{A})).
    \end{align*}
    From this and Lemma~\ref{lem:TiLB},
    we have
    \begin{align*}
        (\Deltar(i))^2 \Tr_{T,i}(A)
        &
        \ge
        \frac{1}{5}
        \ln
        {
        \frac{T}{2 \rbrm[\big]{
        R_T(A) / \Deltar(i) +  R_T(\tilde{A}) / \Deltar(i)
        }}
        }.
    \end{align*}
    From the assumption that $R_T(\hat{A}) \le g(m,n) T^{1 - \epsilon}$ for any $\hat{A}$,
    we have
    \begin{align*}
        \frac{T}{
        R_T(A) / \Deltar(i) +  R_T(\tilde{A}) / \Deltar(i)
        }
        \ge
        \frac{T}{2 g(m, n) T^{1-c} /\Deltar(i) }
        =
        \frac{\Deltar(i) T^c}{2 g(m, n) },
    \end{align*}
    which implies
    \begin{align*}
        \Tr_{T,i}(A)
        \ge
        \frac{1}{5 (\Deltar(i))^2}
        {
            \ln {
                \frac{\Delta_i T^c}{4g(m, n)}
            }
        }.
    \end{align*}
    Consequently,
    we have
    \begin{align*}
        &
        \liminf_{T \rightarrow \infty}
        \frac{R_T(A)}{\ln T}
        =
        \liminf_{T \rightarrow \infty}
        \frac{1}{\ln T}
        \rbrm[\bigg]{
            \sum_{\substack{i \in [m]\\ \Deltar(i) > 0}}
            \Tr_{T,i}(A)
            +
            \sum_{\substack{j \in [n]\\ \Deltac(j) > 0}}
            \Deltac(j)
            \Tc_{T,j}(A)
        }
        \\
        &
        \ge
        \liminf_{T \rightarrow \infty}
        \rbrm[\bigg]{
            \sum_{\substack{i \in [m]\\ \Deltar(i) > 0}}
            \frac{1}{5\Deltar(i)}
            +
            \sum_{\substack{j \in [n]\\ \Deltac(j) > 0}}
            \frac{1}{5\Deltac(j)}
        }
        \rbrm[\bigg]{
            c+
            \frac{1}{\ln T}
            \ln {
                \frac{\Delta_i}{4g(m, n)}
            }
        }
        \\
        &
        =
        \frac{c}{5}
        \rbrm[\bigg]{
        \sum_{\substack{i \in [m]\\ \Deltar(i) > 0}}
        \frac{1}{\Deltar(i)}
        +
        \sum_{\substack{j \in [n]\\ \Deltac(j) > 0}}
        \frac{1}{\Deltac(j)}
        },
    \end{align*}
    which completes the proof.
\end{proof}

\end{document}